\documentclass [ letterpaper, 11pt ]{article}
\usepackage[top=1in, bottom=1in, left=1in, right=1in]{geometry}
\usepackage{amsmath}
\usepackage{amsfonts}
\usepackage{amssymb}
\usepackage{mathrsfs}
\usepackage{amsthm}
\usepackage{mathtools}
\usepackage{graphicx}
\usepackage{enumerate}
\usepackage{epstopdf}
\usepackage{sectsty}
\usepackage{chngcntr}
\usepackage{bbm}
\usepackage[round]{natbib}
\usepackage[toc,page]{appendix}
\usepackage{multirow}
\usepackage{authblk}
\usepackage{setspace}
\usepackage{xcolor}

\newcommand{\mbf}[1]{\boldsymbol{#1}}
\newcommand{\mbb}[1]{\mathbb{#1}}
\newcommand{\mcal}[1]{\mathcal{#1}}
\newcommand{\mrm}[1]{\textrm{#1}}
\newcommand\numberthis{\addtocounter{equation}{1}\tag{\theequation}}
\newcommand{\inprod}[2]{\ensuremath{\left\langle{#1}, {#2}\right\rangle}}
\newcommand{\norm}[1]{\left\|{#1}\right\|}
\newcommand{\abs}[1]{\left|{#1}\right|}

\usepackage[flushleft]{threeparttable}
\usepackage{tabularx,ragged2e,booktabs,caption}

\newtheorem{theorem}{Theorem}
\newtheorem{lemma}[theorem]{Lemma}
\newtheorem{corollary}[theorem]{Corollary}
\newtheorem{prop}[theorem]{Proposition}
\theoremstyle{definition}

\newtheorem{remark}{Remark}

\title{A Sieve Quasi-likelihood Ratio Test for Neural Networks with Applications to Genetic Association Studies}
\author[1]{Xiaoxi Shen}
\author[2]{Chang Jiang}
\author[3]{Lyudmila Sakhanenko}
\author[2]{Qing Lu}
\affil[1]{Department of Mathematics, Texas State University, San Marcos, TX, USA}
\affil[2]{Department of Biostatistics, University of Florida, Gainesville, FL, USA}
\affil[3]{Department of Statistics and Probability, Michigan State University, East Lansing, MI, USA}
\date{}                    
\begin{document}
	\maketitle
	
	\begin{abstract}
		Neural networks (NN) play a central role in modern Artificial intelligence (AI) technology and has been successfully used in areas such as natural language processing and image recognition. While majority of NN applications focus on prediction and classification, there are increasing interests in studying statistical inference of neural networks. The study of NN statistical inference can enhance our understanding of NN statistical proprieties. Moreover, it can facilitate the NN-based hypothesis testing that can be applied to hypothesis-driven clinical and biomedical research. In this paper, we propose a sieve quasi-likelihood ratio test based on NN with one hidden layer for testing complex associations. The test statistic has asymptotic chi-squared distribution, and therefore it is computationally efficient and easy for implementation in real data analysis. The validity of the asymptotic distribution is investigated via simulations. Finally, we demonstrate the use of the proposed test by performing a genetic association analysis of the sequencing data from Alzheimer's Disease Neuroimaging Initiative (ADNI).\\
		
	\end{abstract}

	\section{Introduction}
	With the advance of science and technology, we are now in the era of the fourth industrial revolution. One of the key drivers of the fourth industrial revolution is artificial intelligence (AI). Deep neural networks play a critical role in AI and have achieved great success in many fields such as natural language processing and imaging recognition. While great attention has been given to applications of neural works (NN), limited studies have been focus on its theoretical properties and statistical inference, which hinders its application to hypothesis-driven clinical and biomedical research. The study of NN statistical inference can improve our understanding of NN properties and facilitate hypotheses testing using NN. Nevertheless, it is challenging to study NN statistical inference. For instance, it has been pointed out in \citet{fukumizu1996regularity} and \citet{fukumizu2003likelihood} that the parameters in a neural network model are unidentifiable so that classical tests (e.g., Wald test and likelihood ratio test) cannot be used because unidentifaibility of parameters leads to inconsistency of the nonlinear least squares estimators \citep{wu1981asymptotic}.
	
	Many existing literature on NN, such as \citet{shen2021goodness}, \citet{shen2019asymptotic}, \citet{horel2020significance}, \citet{schmidt2020nonparametric}, and \citet{chen1999improved}, are based on the framework of nonparametric regression. It has been shown in \citet{chen1999improved} that the rate of convergence for neural network estimators is $\norm{\hat{f}_n-f_0}=\mcal{O}_p\left((n/\log n)^{-\frac{1+1/d}{4(1+1/(2d))}}\right)$ for sufficiently smooth $f_0$, so one of the advantages of neural networks compared with commonly used in nonparametric regression methods (e.g., Nadaraya-Watson estimator and spline regression) is that neural network estimators can avoid the curse of dimensionality in terms of rate of convergence. 
	
	There are increasing interests in studying hypothesis testing based on neural networks. Recently, \citet{shen2019asymptotic} established asymptotic theories for neural networks, which can be used to perform a nonparametric hypothesis on the true function. \citet{horel2020significance} used a Linderberg-Feller type central limit theorem for random process and second order functional delta method to construct test statistic to perform significance tests on input features. However, the asymptotic distribution of the test statistic is complex, making it difficult to obtain the critical value. \citet{shen2021goodness} proposed a goodness of fit type test based on neural networks. The test statistic is based on comparing the mean squared error values of two neural networks built under the null and the alternative hypothesis. The test statistic has an asymptotic normal distribution, and hence it can be easily used in practice. However, constructing the test statistic requires a random split of the data, which can lead to a potential power loss. In this paper, we propose a sieve quasi likelihood ratio (SQLR) test based on neural networks. Similar to \citet{shen2021goodness}, the test statistic has an asymptotic chi-squared distribution, which facilitate its use in practice. Compared with the goodness of fit test in \citet{shen2021goodness}, the SQLR test does note require data splitting, but requires continuous random input features.

	The rest of the paper is organized as follows. Section 2 provides the general results of the sieve quasi-likelihood ratio test under the setup of nonparametric regressions. In section 3, we apply the general theories to neural networks so that significance tests based on neural networks can be performed. We investigate the validity of the theories via simple simulations in section 4, followed by a real data application to genetic association analysis of the sequencing data from Alzheimer's Disease Neuroimaging Initiative (ADNI) in section 5. The proofs of the main results in are given in the supplementary materials.

	\section{Sieve Quasi-Likelihood Ratio Test}
	Consider the classical setting of a nonparametric regression model under the random design,
	$$
	Y_i=f_0(\mbf{X}_i)+\epsilon_i,\quad i=1,\ldots,n,
	$$
	where the covariates $\mbf{X}_1,\ldots,\mbf{X}_n\in\mbb{R}^d$ are assumed to be i.i.d. from a distribution $P$, and $\epsilon_1,\ldots,\epsilon_n$ are i.i.d. random errors with $\mbb{E}[\epsilon]=0$. $Y_1,\ldots,Y_n$ are the responses, which are continuous random variables. The true functions $f_0$ is assume to be in $\mcal{F}\subset C(\mcal{X})$, where $\mcal{X}\subset\mbb{R}^d$ is a compact subset. For simplicity, we take $\mcal{X}=[-1,1]^d$. The norm considered on $\mcal{F}$ is the $L_2$-norm $\|f\|=\left(\int_{\mcal{X}}|f|^2\mrm{d}P\right)^{1/2}$. We further assume that $\|\epsilon\|_{p,1}=\int_0^\infty(\mbb{P}(|\epsilon|>t))^{1/p}\mrm{d}t<\infty$ for some $p\geq2$. Such a assumption is also considered in \citet{han2019convergence} and is necessary to obtain the desired convergence rate.
	
	The approximate sieve exremum estimator $\hat{f}_n$ based on $\mcal{F}_n$ is defined as
	$$
	\mbb{Q}_n(\hat{f}_n)\leq\inf_{f\in\mcal{F}_n}\mbb{Q}_n(f)+\mcal{O}_p(\eta_n),
	$$
	where $\mbb{Q}_n(f)$ is the classical sample squared error loss function
	$$
	\mbb{Q}_n(f)=\frac{1}{n}\sum_{i=1}^n(Y_i-f(\mbf{X}_i))^2.
	$$
	We assume that $\bigcup_{n=1}^\infty\mcal{F}_n$ is uniformly dense in $\mcal{F}$, that is, for each $f\in\mcal{F}$, there exists $\pi_nf\in\mcal{F}_n$ such that $\sup_{\mbf{x}\in\mcal{X}}|\pi_nf(\mbf{x})-f(\mbf{x})|\to0$ as $n\to\infty$. For simplicity, we assume that the sieve space $\mcal{F}_n$ is countable to avoid additional technical issue on measurability.
	
	The null hypothesis of the sieve quasi-likelihood ratio test is $H_0:\phi(f_0)=0$, which is the same as the one proposed in \citet{shen2005sieve}. We define the sieve quasi-likelihood ratio statistic as
	$$
	LR_n=n\left(\inf_{f\in\mcal{F}_n^0}\mbb{Q}_n(f)-\inf_{f\in\mcal{F}_n}\mbb{Q}_n(f)\right),
	$$
	where $\mcal{F}_n^0$ is the null sieve space given by 
	$$
	\mcal{F}_n^0=\left\{f\in\mcal{F}_n:\phi(f)=0\right\}.
	$$
	Similar to the definition of $\hat{f}_n$, we denote the approximate sieve extremum estimator under $H_0$ by $\hat{f}_n^0$, which  satisfies
	$$
	\mbb{Q}_n(\hat{f}_n^0)\leq\inf_{f\in\mcal{F}_n^0}\mbb{Q}_n(f)+\mcal{O}_p(\eta_n).
	$$
	According to \citet{shen1997methods} and  \citet{shen2005sieve}, we assume that the functional $\phi:\mcal{F}\to\mbb{R}$ has the following smoothness property: for any $f\in\mcal{F}_n$,
	\begin{equation}\label{Eq: Smoothness of Functional}
		|\phi(f)-\phi(f_0)-\phi_{f_0}'[f-f_0]|\leq u_n\|f-f_0\|^\omega,\quad\mrm{as}\quad\|f-f_0\|\to0,
	\end{equation}
	where $\phi_{f_0}'[f-f_0]$ is defined as $\lim_{t\to0}[\phi(f(f_0,t))-\phi(f_0)]/t$ with $f(f_0,t)$ being a path in $t$ connecting $f_0$ and $f$ such that $f(f_0,0)=f_0$ and $f(f_0,1)=f$. $\omega>0$ is the degree of smoothness of $\phi$ at $f_0$,  $\phi_{f_0}'[f-f_0]$ is linear in $f-f_0$, and
	$$
	\|\phi_{f_0}'\|=\sup_{\substack{f\in\mcal{F} \\ \|f-f_0\|>0}}\frac{|\phi_{f_0}'[f-f_0]|}{\|f-f_0\|}<\infty.
	$$
	Then $\phi_{f_0}'$ is a bounded linear functional on $\bar{V}_{f_0}$, which is the completion of $\mrm{span}\{f-f_0:f\in\mcal{F}\}\subset L_2(\mcal{X},\mcal{A},P)$. From the Riesz representation theorem, there exists $v^*\in\bar{V}_{f_0}$
	$$
	\phi_{f_0}'[f-f_0]=\langle f-f_0,v^*\rangle=\int (f-f_0)v^*\mrm{d}P.
	$$

	Let $\rho_n$ be the rate of convergence for $\hat{f}_n$, that is, $\|\hat{f}_n-f_0\|=\mcal{O}_p(\rho_n)$. Let $\delta_n$ be a sequence converging to 0 with $\delta_n=\mcal{O}_p(n^{-1/2})$. For $f\in\{f\in\mcal{F}_n:\|f-f_0\|\leq\rho_n\}$, we define
	$$
	\tilde{f}_n(f)=f+\delta_n u^*,
	$$
	where $u^*=\pm v^*/\norm{v^*}^2$. The main result relies on the following conditions:
	\begin{itemize}
		\item [(C1)] (\textbf{Sieve Space}) Suppose that $\mcal{F}_n$ is uniformly bounded. 
		%and is star-shaped around 0, that is if $f\in\mcal{F}_n$, then $\alpha f\in\mcal{F}_n$ for all $\alpha\in[0,1]$. 
		Moreover, assume that there exists a non-increasing continuous function $H(u)$ of $u>0$ such that
		$$
		\log N(u,\mcal{F}_n,L_2(\mbb{P}_n))\leq H(u)\mrm{ for all }u>0,
		$$
		and
		$$
		\int_0^1 H^{1/2}(u)\mrm{d}u<\infty.
		$$
		\item [(C2)] (\textbf{Rate of Convergence}) The rate of convergence $\rho_n$ satisfies $o(n^{-1/4})=\rho_n\gtrsim n^{-\frac{1}{2}+\frac{1}{2p}}$ and
		$$
		n\rho_n^2\geq 2H(\rho_n)\mrm{ and }H(\rho_n)\to\infty\mrm{ as }n\to\infty.
		$$
		\item [(C3)] (\textbf{Approximation Error}) 
		$$
		\sup_Q\sup_{\substack{f\in\mcal{F}_n \\ \|f-f_0\|\leq\rho_n}}\norm{\pi_n\tilde{f}_n(f)-\tilde{f}_n(f)}_{L_2(Q)}=o_p(\rho_n^{-1}\delta_n^2),
		$$
		where the supremum is taken over all probability measures $Q$ on $(\mcal{X},\mcal{A})$ and
		$$
		\sup_{\substack{f\in\mcal{F}_n \\ \norm{f-f_0}\leq\rho_n}}n^{-1}\sum_{i=1}^n\epsilon_i\left(\pi_n\tilde{f}_n(f)(\mbf{X}_i)-\tilde{f}_n(f)(\mbf{X}_i)\right)=o_p(\delta_n^2).
		$$
	\end{itemize}

	\begin{remark}
		\begin{enumerate}[(i)]
			\item Condition (C1) requires that the sieve space $\mcal{F}_n$ is not too complex. The complexity measure in terms of the entropy number is common in theory of nonparametric regression (see \citet{van1987new}, \citet{van1988regression} and \citet{van1990estimating}). 
			
			\item Condition (C2) is on the rate of convergence of sieve estimators. To obtain the desired result, the convergence rate cannot be too slow. Together with (C1), we can derive the uniform law of large numbers for empirical $L_2$ norm, as given in \citet{geer2000empirical}.
			
			\item The conditions on the approximation errors in the setting of nonparametric regression are given in condition (C3). These two requirements are special cases of the ones given in \citet{shen1997methods}.
		\end{enumerate}
	\end{remark}
	
	\begin{theorem}\label{Thm: main result}
		Suppose $\eta_n=o(\delta_n^2)$, and under (C1)-(C3),
		$$
		\abs{\inprod{\hat{f}_n-f_0}{v^*}_n-n^{-1}\sum_{i=1}^n\epsilon_iv^*(\mbf{X}_i)}= o_p(n^{-1/2}).
		$$
	\end{theorem}
	
	\begin{remark}
		In view of Lemma \ref{Lm: inprod} given in the supplementary materials, the empirical inner product $\inprod{\hat{f}_n-f_0}{v^*}_n$ can be replaced by its population version $\inprod{\hat{f}_n-f_0}{v^*}$.
	\end{remark}
	
	We now state the main theorem for sieve quasi-likelihood ratio statistics. The proof of the theorem follows the same steps as those in \citet{shen2005sieve} and are given in the Supplementary materials.
	
	\begin{theorem}\label{Thm: SQLR}
		Under $H_0$ and (C1)-(C3), suppose that $\eta_n=o(\delta_n^2)$, $u_n\rho_n^\omega=o(n^{-1/2})$, and $\sup_{\mbf{x}\in\mcal{X}}\abs{v^*(\mbf{x})}<\infty$, we have
		$$
		\frac{n}{\sigma^2}\left[\mbb{Q}_n(\hat{f}_n^0)-\mbb{Q}_n(\hat{f}_n)\right]\xrightarrow{d}\chi_1^2,
		$$
		where $\sigma^2=\mbb{E}[\epsilon^2]$.
	\end{theorem}
	
	In practice, $\sigma^2$ is rarely known apriori. A simple application of Slutsky's theorem yields the following corollary, which shows that we can replace $\sigma^2$ with any consistent estimator of $\hat{\sigma}_n^2$. A straightforward consistent estimator for $\sigma^2$ is given by $\hat{\sigma}_n^2=n^{-1}\sum_{i=1}^n\left(Y_i-\hat{f}_n^0(\mbf{X}_i)\right)^2$.
	
	\begin{corollary}\label{Cor: SQLR}
		Under the conditions of Theorem \ref{Thm: SQLR},
		$$
		\frac{n}{\hat{\sigma}_n^2}\left[\mbb{Q}_n(\hat{f}_n^0)-\mbb{Q}_n(\hat{f}_n)\right]\xrightarrow{d}\chi_1^2,
		$$
		where $\hat{\sigma}_n^2$ is any consistent estimator of $\sigma^2$.
	\end{corollary}
	
	\section{An Application to Neural Networks}
	We first introduce the notations to be used in this section. $\mbf{e}_i=(0,\ldots,0,1,0,\ldots,0)$ where 1 appears at the $i$th position. We use $\mbb{Z}_+$ to the set of non-negative integers and use $\mbf{\beta}=(\beta_1,\ldots,\beta_d)\in\mbb{Z}_+^d$ to denote a multi-index. Moreover, we set $\abs{\mbf{\beta}}=\sum_{i=1}^d\abs{\beta_i}$ and $\mbf{x}^{\mbf{\beta}}=x_1^{\beta_1}\cdots x_d^{\beta_d}$ for any $\mbf{x}=(x_1,\ldots,x_d)^T\in\mbb{R}^d$. For a differentiable function $u$ on $\mcal{X}$, we set
	$$
	D^{\mbf{\beta}} u=\frac{\partial^{|\mbf{\beta}|}u}{\partial \mbf{x}^{\mbf{\beta}}}=\frac{\partial^{|\mbf{\beta}|}u}{\partial x_1^{\beta_1}\cdots\partial x_d^{\beta_d}}.
	$$
	%$\omega(f,\delta)$ is used to denote the modulus of continuity for a function $f\in C(\mcal{X})$, i.e.,
	%$$
	%\omega(f,\delta)=\sup\{|f(x)-f(y)|:x,y\in\mcal{X},|x-y|\leq\delta\}.
	%$$
	
	One of our goals in this paper is to establish a sieve quasi-likelihood ratio test for neural network estimators. Specifically, for a given $k\leq d$, let $\mbf{X}=(X^{(1)},\ldots,X^{(k)},X^{(k+1)},\ldots,X^{(d)})^T\in\mbb{R}^d$ and the null hypothesis of interest be 
	$$
	H_0: X^{(1)},\ldots,X^{(k)}\mrm{ are not significantly associated with Y}.
	$$
	Different from linear regression, in which the hypothesis can be easily translated into testing whether the corresponding regression coefficients are zero, testing significance of an association in nonparametric regression is more complicated. From \citet{chen1999improved} and \citet{horel2020significance}, testing $H_0$ in the nonparametric setting is equivalent to test whether the corresponding partial derivatives are zeros, or equivalently, to test
	$$
	H_0:\sum_{i=1}^k\int\left(D^{\mbf{e}_i}f_0(\mbf{x})\right)^2\mrm{d}P(\mbf{x})=0.
	$$
	Hence, we assume that the true function $f_0$ is a smooth function. Specifically, we consider the Barron class $\mathscr{B}^s:=\{f:\mcal{X}\to\mbb{R}|\norm{f}_{\mathscr{B}^s}\leq B\}$ for some integer $s\geq1$ and some fixed constant $B$, as considered in \citet{siegel2020approximation}. Here 
	$$
	\norm{f}_{\mathscr{B}^s}=\int_{\mbb{R}^d}(1+\abs{\omega})^s|\hat{f}(\omega)|\mrm{d}\omega,
	$$
	and $\hat{f}(\omega)$ is the Fourier transform of $f$. As shown in \citet{siegel2020approximation}, $\mathscr{B}^s\subset H^s(\mcal{X})=\{f:\mcal{X}\to\mbb{R}|\norm{D^{\mbf{\alpha}} f}<\infty,\mrm{ for all }0\leq\abs{\mbf{\alpha}}\leq s\}$ and $H^{\left\lfloor\frac{d}{2}\right\rfloor+2}(\mcal{X})\subset\mathscr{B}^1$. In what follows, we will take $\mcal{F}=C^{m_0}(\mcal{X})$ with $m_0=\left\lfloor\frac{d}{2}\right\rfloor+2$.
	%By the same arguments in \citet{li1996simultaneous}, without loss of generality, we may consider $\mcal{X}$ as the standard simplex in $\mbb{R}^d$, i.e.,
	%$$
	%\mcal{X}=\left\{(x^{(1)},\ldots,x^{(d)}):\sum_{i=1}^dx^{(i)}\leq1, 0\leq x^{(1)},\ldots,x^{(d)}\leq 1\right\}.
	%$$
	
	The functional $\phi$ from the general result in section 2 is given by
	\begin{align*}
		\phi: & \mcal{F}\to\mbb{R}\\
		& f\mapsto\phi(f)=\sum_{i=1}^k\int\left(D^{\mbf{e}_i}f_0(\mbf{x})\right)^2\mrm{d}P(\mbf{x}).\numberthis\label{Eq: significance functional}
	\end{align*}
	The directional derivative $\phi_{f_0}'$ evaluated at ``direction" $h$ can be calculated straightforwardly. For the sieve space, we use the class of neural networks with one hidden layer and sigmoid activation function $\sigma(x)=(1+e^{-x})^{-1}$.
	\begin{align*}
		\mcal{F}_{r_n} & =\left\{\alpha_0+\sum_{j=1}^{r_n}\alpha_j\sigma\left(\mbf{\gamma}_j^T\mbf{x}+\gamma_{0,j}\right):\mbf{\gamma}_j\in\mbb{R}^d, \alpha_j, \gamma_{0,j}\in\mbb{R},\right.\\
		& \qquad\sum_{j=0}^{r_n}|\alpha_j|\leq V\mrm{ for some }V>4\left.\mrm{ and }\max_{1\leq j\leq r_n}\sum_{i=0}^d|\gamma_{i,j}|\leq M\mrm{ for some }M>0\right\},\numberthis\label{Eq: NNSieve}
	\end{align*}
	where $r_n\uparrow\infty$ as $n\to\infty$. In view of \citet{barron1993universal}, $\mcal{F}_{r_n}$ is $L_2$-dense in $\mcal{F}$.
	
	Based on the general results in the previous section, the function $\phi$ needs to be smooth so that the sieve quasi-likelihood ratio statistic follows a chi-squared asymptotic distribution. The following propositions guarantee the satisfaction of conditions on $\phi$ in the general theory.
	
	\begin{prop}
		Let $\phi$ be the same function given in (\ref{Eq: significance functional}), then, for any $h\in \bar{V}_{f_0}$,
		$$
		\phi_{f_0}'[h]=2\sum_{i=1}^k\int D^{\mbf{e}_i}f_0(\mbf{x})D^{\mbf{e}_i}h(\mbf{x})\mrm{d}P(\mbf{x}).
		$$
		Moreover, $\phi_{f_0}'$ is a bounded linear functional on $\bar{V}_{f_0}$. 
	\end{prop}
	
	\begin{proof}
		By definition,
		\begin{align*}
			\phi_{f_0}'[h]	& =\lim_{t\to0}\frac{\phi(f_0+th)-\phi(f_0)}{t}\\
			& =\sum_{i=1}^k\lim_{t\to0}\int\frac{\left(D^{\mbf{e}_i}(f_0+th)(\mbf{x})\right)^2-\left(D^{\mbf{e}_i}f_0(\mbf{x})\right)^2}{t}\mrm{d}P(\mbf{x})\\
			& =2\sum_{i=1}^k\int D^{\mbf{e}_i}f_0(\mbf{x})D^{\mbf{e}_i}h(x)\mrm{d}P(\mbf{x})+\sum_{i=1}^k\lim_{t\to0}t\int\left(D^{\mbf{e}_i}h(\mbf{x})\right)^2\mrm{d}P(\mbf{x})\\
			& =2\sum_{i=1}^k\int D^{\mbf{e}_i}f_0(\mbf{x})D^{\mbf{e}_i}h(\mbf{x})\mrm{d}P(\mbf{x}).
		\end{align*}
		For the second claim, linearity follows directly from the definition of $\phi_{f_0}'$. Boundedness follows from the H\"older's inequality by noting that
		\begin{align*}
			\sup_{\substack{h\in\mcal{F} \\ \norm{h}=1}}|\phi_{f_0}'[h]|	& \leq2\sum_{i=1}^k\sup_{\substack{h\in\mcal{F} \\ \norm{h}=1}}\abs{\int D^{\mbf{e}_i}f_0(\mbf{x})D^{\mbf{e}_i}h(\mbf{x})\mrm{d}P(\mbf{x})}\\
			& \leq 2\sum_{i=1}^k\norm{D^{\mbf{e}_i}f_0}\sup_{\substack{h\in\mcal{F} \\ \norm{h}=1}}\norm{D^{\mbf{e}_i}h}\\
			& \lesssim_{s,\mcal{X}} 2kB^2<\infty.
		\end{align*}
	\end{proof}

	We now impose the following condition on the distribution $P$.
	\begin{itemize}
		\item [(C4)] Suppose that $P\ll\lambda$, where $\lambda$ is the Lebesgue measure on $\mbb{R}^d$. Let
		$$
		\varphi(\mbf{x})=\frac{\mrm{d}P}{\mrm{d}\lambda}(\mbf{x})\geq0.
		$$
		Moreover, we assume that $\varphi=0$ on $\partial\mcal{X}$,  $\varphi\in L^\infty(\mcal{X})$, and $\log\varphi\in C^1(\mcal{X})$.
	\end{itemize}
	
	\begin{prop}
		Under (C4), the Riesz representor $v^*$ for the bounded linear functional $\phi_{f_0}'$ is given by
		$$
		v^*=-2\sum_{i=1}^k\left(D^{2\mbf{e}_i}f_0+D^{e_i}f_0D^{\mbf{e}_i}\log\varphi\right).
		$$
	\end{prop}
	
	\begin{proof}
		Define $g:\mcal{X}\to\mbb{R}$ and $\mbf{F}:\mcal{X}\to\mbb{R}^d$ as
		\begin{align*}
			\mbf{F} & =(0,\ldots,0,\underbrace{h}_{i\mrm{th position}},0,\ldots,0)^T and\\
			g	& =\varphi D^{\mbf{e}_i}f_0,
		\end{align*}
		then $g\nabla\cdot \mbf{F}=\varphi D^{\mbf{e}_i}f_0D^{\mbf{e}_i}h$. Let $\mbf{n}$ be the unit outward normal to $\partial\mcal{X}$. Given the integration by parts formula and the fact that $\varphi=0$ on $\partial\mcal{X}$, we have
		\begin{align*}
			\int_{\mcal{X}}\varphi D^{\mbf{e}_i}f_0D^{\mbf{e}_i}h\mrm{d}\mbf{x}	& =\int_{\mcal{X}} g\nabla\cdot \mbf{F}\mrm{d}\mbf{x}\\
			& =-\int_{\mcal{X}}\nabla g\cdot \mbf{F}\mrm{d}\mbf{x}+\int_{\partial\mcal{X}}g\mbf{F}\cdot \mbf{n}\mrm{d}S\\
			& =-\int_{\mcal{X}}\nabla g\cdot \mbf{F}\mrm{d}\mbf{x}\\
			& =-\int_{\mcal{X}} h\left(D^{\mbf{e}_i}\varphi D^{\mbf{e}_i}f_0+\varphi D^{2\mbf{e}_i}f_0\right)\mrm{d}\mbf{x}\\
			& =-\int_{\mcal{X}} h\left(D^{\mbf{e}_i}\log\varphi D^{\mbf{e}_i}f_0+D^{2\mbf{e}_i}f_0\right)\mrm{d}P(\mbf{x})\\
			& =\inprod{h}{-\left(D^{2\mbf{e}_i}f_0+D^{2\mbf{e}_i}f_0D^{\mbf{e}_i}\log\varphi\right)},
		\end{align*}
		Based on the given assumptions, we know that $-\left(D^{2\mbf{e}_i}f_0+D^{2\mbf{e}_i}f_0D^{\mbf{e}_i}\log\varphi\right)\in C(\mcal{X})\subset\bar{V}_{f_0}$. Therefore,
		\begin{align*}
			\phi_{f_0}'[h]	& =2\sum_{i=1}^k\int_{\mcal{X}}D^{\mbf{e}_i}f_0D^{\mbf{e}_i}h\mrm{d}P\\
			& =2\sum_{i=1}^k\int_{\mcal{X}}\varphi D^{\mbf{e}_i}f_0D^{\mbf{e}_i}h\mrm{d}\mbf{x}\\
			& =\inprod{h}{-2\sum_{i=1}^k\left(D^{2\mbf{e}_i}f_0+D^{2\mbf{e}_i}f_0D^{\mbf{e}_i}\log\varphi\right)}.
		\end{align*}
	\end{proof}
	
	Before we bound the remainder error of the first order functional Taylor expansion, we provide a bound for higher order derivatives of a neural network.
	\begin{prop}\label{Prop: higher order deriv NN}
		Let $m$ be a non-negative integer. For any $f\in\mcal{F}_{r_n}$ and any multi-index $\mbf{\beta}$ with $|\mbf{\beta}|=m$,
		$$ 
		\sup_{\mbf{x}\in\mcal{X}}\abs{D^{\mbf{\beta}} f(\mbf{x})}\leq VM^mm!.
		$$
	\end{prop}
	
	\begin{proof}
		As $f\in\mcal{F}_{r_n}$, $f$ can be represented by
		$$
		f(\mbf{x})=\alpha_0+\sum_{j=1}^{r_n}\alpha_j\sigma\left(\mbf{\gamma}_j^T\mbf{x}+\gamma_{0,j}\right).
		$$
		A simple calculation yields
		$$
		D^{\mbf{\beta}} f(\mbf{x})=\sum_{j=1}^{r_n}\alpha_j\mbf{\gamma}_j^{\mbf{\beta}}\sigma^{(m)}\left(\mbf{\gamma}_j^T\mbf{x}+\gamma_{0,j}\right),
		$$
		where $\sigma^{(m)}(\cdot)$ represents the $m$th derivative of $\sigma$. According to \citet{minai1993derivatives}, we have
		$$
		\sigma^{(m)}(z)=\sum_{a=1}^m(-1)^{a-1}C_a^{(m)}\sigma^a(z)\left(1-\sigma(z)\right)^{m+1-a},
		$$
		where
		\begin{align*}
			C_a^{(m)}	& =0,\quad \forall m,\mrm{ if }a<1\mrm{ or }m\leq0\mrm{ or }s>m,\\
			C_1^{(1)}	& =1,\\
			C_a^{(m)}	& =aC_a^{(m-1)}+(m+1-a)C_{a-1}^{(m-1)}.
		\end{align*}
		Therefore,
		\begin{align*}
			\sup_{\mbf{x}\in\mcal{X}}|D^{\mbf{\beta}} u(\mbf{x})|	& =\sup_{x\in\mcal{X}}\abs{\sum_{j=1}^{r_n}\alpha_j\mbf{\gamma}_j^{\mbf{\beta}}\sigma^{(m)}\left(\mbf{\gamma}_j^T\mbf{x}+\gamma_{0,j}\right)}\\
			& \leq\sum_{j=1}^{r_n}|\alpha_j|\norm{\mbf{\gamma}_j}_{\ell_1}^m\sup_{x\in\mcal{Z}}\abs{\sigma^{(m)}\left(\mbf{\gamma}_j^T\mbf{x}+\gamma_{0,j}\right)}\\
			& \leq VM^m\sup_{z\in\mbb{R}}\abs{\sigma^{(m)}(z)}\\
			& =VM^m\sup_{z\in\mbb{R}}\abs{\sum_{a=1}^m(-1)^{a-1}C_a^{(m)}\sigma^a(z)\left(1-\sigma(z)\right)^{m+1-a}}\\
			& \leq VM^m\sum_{a=1}^mC_a^{(m)}\\
			& \overset{(i)}{=}VM^mm!,
		\end{align*}
		where (i) follows from Proposition \ref{Prop: sigmoid Coef} in the supplementary material.
	\end{proof}

	\begin{lemma}[Rate of Convergence of Neural Network Sieve Estimators]\label{LM: RoC NN Sieve}
		\begin{enumerate}[(i)]
			\item The sieve space $\mcal{F}_{r_n}$ satisfies (C1).
			
			\item Suppose that $r_n^{2+1/d}\log^2r_n=\mcal{O}(n)$, then the rate of convergence $\rho_n$ of neural network sieve estimators is
			$$
			\rho_n=\mcal{O}\left(\left(\frac{n}{\log^2n}\right)^{-\frac{1+1/d}{4(1+1/(2d))}}\right),
			$$
			and $\rho_n$ satisfies (C2).
		\end{enumerate}
	\end{lemma}
	
	\begin{proof}
		\begin{enumerate}[(i)]
			\item From Theorem 14.5 in \citet{anthony2009neural}, we have
			\begin{align*}
				N(u,\mcal{F}_{r_n},\norm{\cdot}_\infty)	& \leq\left(\frac{4e[r_n(d+2)+1]\left(\frac{1}{4}V\right)^2}{u\left(\frac{1}{4}V-1\right)}\right)^{r_n(d+2)+1}\\
				& =\left(\frac{e[r_n(d+1)+1]V^2}{u(V-4)}\right)^{r_n(d+2)+1},
			\end{align*}
			which implies that
			\begin{align*}
				\log N(u,\mcal{F}_{r_n},L_2(\mbb{P}_n))	& \leq\log N(u,\mcal{F}_{r_n},\norm{\cdot}_\infty)\lesssim_{d,V}(r_n\log r_n)\log\frac{1}{u}.\numberthis\label{Eq: UB of Entropy NN}
			\end{align*}
			Hence, (C1) is satisfied with $H(u)=C_{d,V}\cdot\left(r_n\log r_n\right)\log\frac{1}{u}$ by noting that
			\begin{align*}
				\int_0^1 \log^{1/2}\frac{1}{u}\mrm{d}u	& =\int_0^{1/2}\log^{1/2}\frac{1}{u}\mrm{d}u+\int_{1/2}^1\log^{1/2}\frac{1}{u}\mrm{d}u\\
				& \leq \left.u\log^{1/2}\frac{1}{u}\right|_0^{1/2}+\frac{1}{2}\int_0^{1/2}\log^{-1/2}\frac{1}{u}\mrm{d}u+\frac{1}{2}\log^{1/2}2\\\\
				& \leq\log^{1/2}2+\frac{1}{4}\log^{-1/2}2<\infty.
			\end{align*}
			
			\item Note that, for $\delta\leq1/e$,
			\begin{align*}
				\int_0^\delta \log^{1/2}\frac{1}{u}\mrm{d}u	& =\left.u\log^{1/2}\frac{1}{u}\right|_0^\delta+\int_0^\delta \frac{1}{2}\log^{-1/2}\frac{1}{u}\mrm{d}u\\
				& \leq\delta\log^{1/2}\frac{1}{\delta}+\delta\log^{-1/2}\frac{1}{\delta}\\
				& \lesssim \delta\log^{1/2}\frac{1}{\delta}.
			\end{align*}
			Let $\phi_n(\delta)=(r_n\log r_n)^{1/2}\delta\log^{1/2}\frac{1}{\delta}$. Clearly, $\delta^{-\alpha}\phi_n(\delta)$ is decreasing on $(0,\infty)$ for $1\leq\alpha<2$. Note that
			\begin{equation}\label{Eq: RoC condition}
				\rho_n^{-2}\phi_n(\rho_n)\lesssim\sqrt{n}\Leftrightarrow \rho_n^{-1}\log^{1/2}\rho_n^{-1}\lesssim\left(\frac{n}{r_n\log r_n}\right)^{1/2}.
			\end{equation}
			It follows from \citet{makovoz1996random} that $\norm{f_0-\pi_nf_0}\leq r_n^{-1/2-1/(2d)}$. By taking $\rho_n=r_n^{-1/2-1/(2d)}$ in (\ref{Eq: RoC condition}), we obtain the following governing inequality:
			$$
			r_n^{1+\frac{1}{2d}}\log r_n\lesssim_d\sqrt{n}\Leftrightarrow r_n^{2+1/d}\log^2r_n=\mcal{O}(n).
			$$
			Given that $r_n=\left(\frac{n}{\log^2n}\right)^{\frac{1}{2+\frac{1}{d}}}$, we have
			$$
			\rho_n=\mcal{O}\left(\left(\frac{n}{\log^2n}\right)^{-\frac{1+1/d}{4(1+1/(2d))}}\right).
			$$
			To show that $\rho_n$ satisfies condition (C2), we note that $H(\rho_n)\to\infty$ as long as $\rho_n\to0$ as $n\to\infty$. The governing inequality is certainly satisfied based on the previous arguments. We also note that
			\begin{align*}
				\rho_n	&	=\left(\frac{n}{\log^2 n}\right)^{-\frac{1+1/d}{4(1+1/(2d))}}\\
				& = n^{-\frac{1}{4}}n^{\frac{1}{4}-\frac{1+1/d}{4(1+1/(2d))}}(\log^2 n)^{\frac{1+1/d}{4(1+1/(2d))}}\\
				& =n^{-\frac{1}{4}}n^{-\frac{1}{4}\frac{1/(2d)}{4(1+1/(2d))}}(\log^2 n)^{\frac{1+1/d}{4(1+1/(2d))}}\\
				& =o(n^{-1/4}).
			\end{align*}
			On the other hand, we have
			\begin{align*}
				\rho_n	& \geq n^{-\frac{1+1/d}{4(1+1/(2d))}}=n^{-\frac{1}{2}}n^{\frac{1}{2}-\frac{1+1/d}{4(1+1/(2d))}}=n^{-\frac{1}{2}}n^{\frac{1}{2\left(2+\frac{1}{d}\right)}}\geq n^{-\frac{1}{2}+\frac{1}{2p}},
			\end{align*}
			where the last inequality follows from the assumption $p\geq 2+1/d$.
		\end{enumerate}
	\end{proof}
	
	\begin{remark}
		The rate of convergence we obtained has an additional $\log n$ term in the denominator compared with the results in \citet{chen1998sieve}, but this has little effect on the main result.
	\end{remark}
	
	\begin{prop}
		Under (C4) and the assumption of
		\begin{equation}\label{Eq: growth rate}
			r_n^{2+1/d}\log^2 r_n=\mcal{O}(n),
		\end{equation}
		for any $f\in\{f\in\mcal{F}_{r_n}:\norm{f-f_0}\leq\rho_n\}$, we have
		$$
		\abs{\phi(f)-\phi(f_0)-\phi_{f_0}'[f-f_0]}=o(n^{-1/2}).
		$$
	\end{prop}
	
	\begin{proof}
		Note that
		\begin{align*}
			\abs{\phi(f)-\phi(f_0)-\phi_{f_0}'[f-f_0]}	& =\abs{\sum_{i=1}^k\left[\int\left(D^{\mbf{e}_i}f(\mbf{x})\right)^2-\left(D^{\mbf{e}_i}f_0(\mbf{x})\right)^2-2D^{\mbf{e}_i}f_0(\mbf{x})D^{\mbf{e}_i}(f-f_0)(\mbf{x})\mrm{d}P(\mbf{x})\right]}\\
			& =\sum_{i=1}^k\int \left(D^{\mbf{e}_i}(f-f_0)(\mbf{x})\right)^2\mrm{d}P(\mbf{x})\\
			& \leq2\sum_{i=1}^k\int\left(D^{\mbf{e}_i}(f-\pi_{r_n}f_0)(\mbf{x})\right)^2\mrm{d}P(\mbf{x})+2\sum_{i=1}^k\int\left(D^{\mbf{e}_i}(\pi_{r_n}f_0-f_0)(\mbf{x})\right)^2\mrm{d}P(\mbf{x}),
		\end{align*}
		where the last inequality follows from the elementary inequality $(a+b)^2\leq 2(a^2+b^2)$ and the triangle inequality. For the second term, it follows from Corollary 1 in \citet{siegel2020approximation} that
		\begin{align*}
			2\sum_{i=1}^k\int\left(D^{\mbf{e}_i}(\pi_{r_n}f_0-f_0)(\mbf{x})\right)^2\mrm{d}P(\mbf{x})	& =2\sum_{i=1}^k\int\left(D^{\mbf{e}_i}(\pi_{r_n}f_0-f_0)(\mbf{x})\right)^2\varphi(\mbf{x})\mrm{d}\mbf{x}\\
			& \leq2\norm{\varphi}_\infty\norm{\pi_{r_n}f_0-f_0}_{H^s(\mcal{X})}^2\\
			& \lesssim_{\mcal{X},d}n^{-1}=o(n^{-1/2}).
		\end{align*}
		For the first term, we use the Gagliardo-Nirenberg interpolation inequality (Theorem 12.87 in \citet{leoni2017first}). For $m>1$ and $\theta=1-\frac{1}{m}$, there exists a constant $C$, which is independent of $f-\pi_{r_n}f_0$, such that
		$$
		\norm{\nabla(f-\pi_{r_n}f_0)}\leq C\norm{f-\pi_{r_n}f_0}^{1-\frac{1}{m}}\norm{\nabla^m(f-f_0)}^{\frac{1}{m}}.
		$$
		It then follows from Proposition \ref{Prop: higher order deriv NN} that
		\begin{align*}
			2\sum_{i=1}^k\int\left(D^{\mbf{e}_i}(f-\pi_{r_n}f_0)(\mbf{x})\right)^2\mrm{d}P(\mbf{x})	& = 2\sum_{i=1}^k\norm{D^{\mbf{e}_i}(f-\pi_{r_n}f_0)}^2\\
			& \leq2\norm{\nabla(f-\pi_{r_n}f_0)}^2\\
			& \leq2C\norm{f-\pi_{r_n}f_0}^{2-\frac{2}{m}}\norm{\nabla^m(f-\pi_{r_n}f_0)}^{\frac{2}{m}}\\
			& \leq2C(2\rho_n)^{2-\frac{2}{m}}\binom{d-1+m}{m}\left(\max_{\mbf{\beta}:\abs{\mbf{\beta}}=m}\sup_{\mbf{x}\in\mcal{X}}\abs{D^{\mbf{\beta}}(f-\pi_{r_n}f_0)}\right)^{\frac{2}{m}}\\
			& \leq 2C(2\rho_n)^{2-\frac{2}{m}}\binom{d-1+m}{m}\left(2VM^mm!\right)^{\frac{2}{m}}\\
			& =8C\binom{d-1+m}{m}V^{\frac{2}{m}}M^2(m!)^{\frac{2}{m}}\rho_n^{2-\frac{2}{m}}.
		\end{align*}
		As we have shown in Lemma \ref{LM: RoC NN Sieve}, under (\ref{Eq: growth rate}), $\rho_n=\left(n/\log^2 n\right)^{-\frac{1+1/d}{4(1+1/(2d))}}$, and then 
		\begin{align*}
			\rho_n^{2-\frac{2}{m}}	& =\left(\frac{n}{\log^2 n}\right)^{-\frac{1+1/d}{2(1+1/(2d))}\left(1-\frac{1}{m}\right)}\\
			& =n^{-\frac{1}{2}}n^{-\frac{1}{2m}\frac{m/(2d)-1-1/d}{1+1/(2d)}}\left(\log^2 n\right)^{-\frac{1+1/d}{2(1+1/(2d))}\left(1-\frac{1}{m}\right)}.
		\end{align*}
		By taking $m>2d+2$, we obtain that
		$$
		2\sum_{i=1}^k\int\left(D^{\mbf{e}_i}(f-\pi_{r_n}f_0)(\mbf{x})\right)^2\mrm{d}P(\mbf{x})\lesssim_{d,V,M} o(n^{-1/2}).
		$$
		Therefore, we obtain that for $f\in\{f\in\mcal{F}_{r_n}:\norm{f-f_0}\leq\rho_n\}$, 
		$$
		\abs{\phi(f)-\phi(f_0)-\phi_{f_0}'[f-f_0]}=o(n^{-1/2}).
		$$
		%	For the first term, it follows from Proposition 4.2 in \citet{li1996simultaneous} that
		%	
		%	\begin{align*}
		%		3\sum_{i=1}^k\int\left(\frac{\partial(f-\mcal{B}_nf)(x)}{\partial x^{(i)}}\right)^2\mrm{d}P(x)	& =3\sum_{i=1}^k\norm{D^{e_i}f-D^{e_i}\mcal{B}_nf}^2\\
		%		& \leq3\sum_{i=1}^k\norm{D^{e_i}f-D^{e_i}\mcal{B}_nf}_\infty^2\\
		%		& \leq24\omega^2\left(D^{e_i}f,\frac{1}{\sqrt{n-1}}\right)+24\omega^2\left(D^{e_i}f,\frac{1}{n}\right)
		%	\end{align*}
		%	As $f\in\mcal{F}_{r_n}$, by Lemma \ref{Lm: Deriv NN-Lip} in the supplementary materials, we can know that
		%	$$
		%	\omega\left(D^{e_i}f,\frac{1}{\sqrt{n-1}}\right)\lesssim_{V,M}\frac{1}{\sqrt{n-1}},\quad\omega\left(D^{e_i}f,\frac{1}{n}\right)\lesssim_{V,M}\frac{1}{n},
		%	$$
		%	which implies that
		%	$$
		%	3\sum_{i=1}^k\int\left(\frac{\partial(f-\mcal{B}_nf)(x)}{\partial x^{(i)}}\right)^2\mrm{d}P(x)=\mcal{O}(n^{-1}).
		%	$$
		%	Similarly, for the third term, as $f_0\in C^m(\mcal{X})$ and $m\geq2$, we can know that $D^{e_i}f_0$ is Lipschitz continuous so that
		%	$$
		%	3\sum_{i=1}^k\int\left(\frac{\partial(\mcal{B}_nf_0-f_0)(x)}{\partial x^{(i)}}\right)^2\mrm{d}P(x)=\mcal{O}(n^{-1}).
		%	$$
		%	Let us now focus on the second term. By (4.4a) in \citet{li1996simultaneous}, we have
		%	\begin{align*}
		%		D^{e_i}\mcal{B}_nf(x)	& =\sum_{|k|\leq n-1}\frac{f((k+e_i)/n)-f(k/n)}{1/n}\binom{n-1}{k}x^k\left(1-\sum_{i=1}^dx^{(i)}\right)^{n-1-|k|}\\
		%		D^{e_i}\mcal{B}_nf_0(x)	& =\sum_{|k|\leq n-1}\frac{f_0((k+e_i)/n)-f_0(k/n)}{1/n}\binom{n-1}{k}x^k\left(1-\sum_{i=1}^dx^{(i)}\right)^{n-1-|k|}.
		%	\end{align*}
		
	\end{proof}
	
	Now we state and prove the asymptotic distribution of the sieve quasi-likelihood ratio statistic.
	\begin{theorem}
		Suppose that $\eta_n=o(\delta_n^2)$ and $\norm{\epsilon}_{p,1}<\infty$ for some $p\geq2+1/d$,  under (\ref{Eq: growth rate}) and $H_0$,
		$$
		\frac{n}{\hat{\sigma}_n^2}\left[\mbb{Q}_n(\hat{f}_n^0)-\mbb{Q}_n(\hat{f}_n)\right]\xrightarrow{d}\chi_1^2,
		$$
		where $\hat{\sigma}_n^2$ is any consistent estimator of $\sigma^2$.
	\end{theorem}
	
	\begin{proof}
		While conditions (C1) and (C2) have been verified in Lemma \ref{LM: RoC NN Sieve}, condition (C3) remains to be verified. According to Theorem 2.1 in \citet{mhaskar1996neural}, we can find vectors $\{\gamma_j\}_{j=1}^{r_n}\in\mbb{R}^d$ and $\{\gamma_{0,j}\}_{j=1}^{r_n}\in\mbb{R}$ such that for any $f\in\mcal{W}^{m_0,2}([-1,1]^d)$, there exists coefficients $\alpha_j(f)$ satisfying
		\begin{equation}\label{Eq: mhaskar approximation}
			\norm{f-\sum_{j=1}^{r_n}\alpha_j(f)\sigma\left(\gamma_j^Tx+\gamma_{0,j}\right)}\lesssim r_n^{-m_0/d}\norm{f}_{\mcal{W}^{m_0,2}([-1,1]^d)}.
		\end{equation}
		In addition, the functionals $\alpha_j$ are continuous linear functionals on $\mcal{W}^{m_0,2}([-1,1]^d)$.

		Based on the results from \citet{goulaouic1971approximation} or \citet{baouendi1974approximation} and Lemma 3.2 in \citet{mhaskar1996neural} we can show that for an analytic function $f$ defined on a compact set $K$, there exist $a>1$ and coefficients $\alpha_j(f)$ such that
		\begin{equation}\label{Eq: Approx NN}
			\norm{f-\sum_{j=1}^{r_n}\alpha_j(f)\sigma\left(\gamma_j^Tx+\gamma_{0,j}\right)}_p\lesssim a^{-r_n^{1/d}},
		\end{equation}
		where $\gamma_j$ and $\gamma_{0,j}$ are the same as those given in (\ref{Eq: mhaskar approximation}). Since $f$ is analytic, for every $f\in\mcal{F}_{r_n}$ with $\norm{f-f_0}\leq\rho_n$, there exists a neural network $\pi_{r_n}f\in\mcal{F}_{r_n}$ with
		$$
		\pi_{r_n}f=\sum_{j=1}^{r_n}\alpha_j(f)\sigma\left(\gamma_j^Tx+\gamma_{0,j}\right),
		$$
		such that $\norm{f-\pi_{r_n}f}_\infty\lesssim a^{-n^{1/d}}$ for some $a>1$. For $f_0+u^*\in C^m(\mcal{X})$, there exists a neural network $\pi_{r_n}(f_0+u^*)\in\mcal{F}_{r_n}$ with
		$$
		\pi_{r_n}(f_0+u^*)=\sum_{j=1}^{r_n}\alpha_j(f_0+u^*)\sigma\left(\gamma_j^Tx+\gamma_{0,j}\right),
		$$
		such that $\norm{f_0+u^*-\pi_{r_n}(f_0+u^*)}_\infty\lesssim\rho_n$. By considering
		$$
		\pi_{r_n}\tilde{f}_n(f)=(1-\delta_n)\pi_{r_n}f+\delta_n\pi_{r_n}(f_0+u^*),
		$$
		it is clear that $\pi_{r_n}\tilde{f}_n(f)\in\mcal{F}_{r_n}$ and
		$$
		\norm{\pi_{r_n}\tilde{f}_n(f)-\tilde{f}_n(f)}_\infty\leq(1-\delta_n)\norm{f-\pi_{r_n}f}_\infty+\delta_n\norm{f_0+u^*-\pi_{r_n}(f_0+u^*)}_\infty=\mcal{O}\left(\delta_n\rho_n\right).
		$$
		Therefore, by choosing $\delta_n=\rho_n^2$ and note that
		$$
		\frac{\rho_n\delta_n}{\rho_n^{-1}\delta_n^2}=\rho_n=o(1),
		$$
		we can know that the first requirement in (C3) is satisfied. For the second requirement, note that if $n$ is large enough,
		\begin{align*}
			& \sup_{\substack{f\in\mcal{F}_n \\ \norm{f-f_0}\leq\rho_n}}n^{-1}\sum_{i=1}^n\epsilon_i\left(\pi_{r_n}\tilde{f}_n(f)(X_i)-\tilde{f}_nf(X_i)\right)\\
			=	& \sup_{\substack{f\in\mcal{F}_n \\ \norm{f-f_0}\leq\rho_n}}(1-\delta_n)n^{-1}\sum_{i=1}^n\epsilon_i\left(\pi_{r_n}f(X_i)-f(X_i)\right)+\delta_nn^{-1}\sum_{i=1}^n\epsilon_i(\pi_{r_n}(f_0+u^*)(X_i)-(f_0+u^*)(X_i))\\
			\leq	& \sup_{\substack{f\in\mcal{F}_n \\ \norm{f-f_0}\leq\rho_n}}\norm{\pi_{r_n}f-f}_\infty n^{-1}\sum_{i=1}^n|\epsilon_i|+\delta_n\norm{\pi_{r_n}(f_0+u^*)-(f_0+u^*)}_\infty\\
			=	& o_p(\rho_n\delta_n^2).
		\end{align*}
		Hence (C3) is satisfied, and the desired claim follows from Corollary \ref{Cor: SQLR}.	
	\end{proof}

	\section{A Simulation Study}
	We conducted a simulation study to investigate the type I error and power performance of our proposed test. The model for generating the simulation data is given as follows:
	$$
	Y_i=8+X_i^{(1)}X_i^{(2)}+\exp\left(X_i^{(3)}X_i^{(4)}\right)+0.1X_i^{(5)}+\epsilon_i,\quad i=1,\ldots,n,
	$$
	where $\mbf{X}_i=\left(X_i^{(1)},X_i^{(2)},X_i^{(3)},X_i^{(4)},X_i^{(5)}, X^{(6)}\right)$, $\mbf{X}_1,\ldots,\mbf{X}_n\sim\mrm{ i.i.d. Unif}([-1,1]^6)$ and $\epsilon_1,\ldots,\epsilon_n\sim\mrm{ i.i.d. }\mcal{N}(0,1)$. Since $X^{(5)}$ is not included in the true model, we use $X^{(5)}$ to investigate whether the SQLR test have good control of type I error, while the other 5 covariates are used for evaluate the power of the proposed test.
	
	A subgradient method discussed in section 7 in \citet{boyd2008note} was applied to obtain a neural network estimate due to the constraints on the sieve space $\mcal{F}_{r_n}$. The step size for the $k$th iteration is chosen to be $0.1/\log(e+k)$ to fit a neural network under the null hypothesis $H_0$, while the step size of $0.1/(300\log(e+k))$ is used under the alternative hypothesis $H_1$. Such choices of step sizes ensure the convergence of the subgradient method. In terms of the structure of the neural networks, we set $r_n=\lfloor n^{1/2}\rfloor$ and $V=1000$ for both neural networks fitted under $H_0$ and $H_1$. When fitting the neural network under $H_0$, the initial value for the weights are randomly assigned. We use the fitted values for the weights from the neural network under $H_0$ as the initial values and set all the extra weights to be zero when we fit the neural network under $H_1$.
	
	Table \ref{Tab: power} summarizes the empirical type I error and the empirical power under various sample sizes for the proposed neural-network-based SQLR test and the linear-regression-based $F$-test after conducting 500 Monte Carlo iterations. Results from table show that both testing procedure can control the empirical type I error well. In terms of empirical power, the $F$-test can only detect the linear component $0.1X^{(5)}$ of the simulated model, while SQLR can detect all the components of the model. Therefore, when nonlinear patterns exist in the underlying function, the SQLR test is anticipate to be more powerful than the $F$-test. Even in the case of linear terms, the performances between the two methods are comparable.
	
	\begin{table}[htbp]
		\centering
		\caption{Empirical type I error rate (for covariate $X^{(6)}$) and empirical powers (for covariates $X^{(1)},\ldots,X^{(5)}$) for the neural-network-based SQLR test and the linear-regression-based $F$-test}\label{Tab: power}
		\begin{tabular}{l|ccccc|ccccc}
			\hline
			& \multicolumn{5}{c}{SQLR} & \multicolumn{5}{c}{$F$-test}\\
			\hline
			Sample Size & 100 & 500 & 1000 & 3000 & 5000 & 100 & 500 & 1000 & 3000 & 5000\\
			\hline
			$X^{(1)}$ & 0.072 & 0.072 & 0.080 & 0.326 & 0.818 & 0.054 & 0.068 & 0.046 & 0.060 & 0.042\\
			$X^{(2)}$ & 0.058	& 0.088 & 0.152 & 0.504 & 0.932 & 0.066 & 0.062 & 0.070 & 0.054 & 0.058\\
			$X^{(3)}$ & 0.052 & 0.062 & 0.104 & 0.308 & 0.812 &  0.050 & 0.048 & 0.058 & 0.078 & 0.060\\
			$X^{(4)}$ & 0.064  & 0.072 & 0.132 & 0.486 & 0.920 & 0.054 & 0.066 & 0.048 & 0.056 & 0.064\\
			$X^{(5)}$ & 0.074 & 0.202 & 0.406 & 0.904 & 0.978 & 0.070 & 0.222 & 0.414 & 0.826 & 0.956\\ 
			$X^{(6)}$ (Type I Error) & 0.054	& 0.058 & 0.046 & 0.042 & 0.060 & 0.046 & 0.054 & 0.038 & 0.032 & 0.054\\
			\hline
		\end{tabular}
	\end{table}
	
	\section{Real Data Applications}
	We conducted two genetic association analyses by applying the proposed sieve quasi-likelihood ratio test based on neural networks to the gene expression data and the sequencing data from Alzheimer's Disease Neuroimaging Initiative (ADNI). Studies have shown that the hippocampus region in brain plays a vital part in memory and learning \citet{mu2011adult} and the change in the volume of hippocampus has a great impact on Alzheimer's disease \citep{schuff2009mri}. For both analyses, we first regress the logarithm of the hippocampus volume on important covariates (i.e, age, gender and education status) and then use the residual obtained as the response variable to fit neural networks. A total of 464 subjects and 15,837 gene expressions were obtained after quality control. 
	
	Under the null hypothesis, the gene is not associated with the response. Therefore, we can use the sample average of the response variable as the null estimator. When we fitted neural networks under the alternative hypothesis, we set the number of hidden units as $r_n=\lfloor n^{1/2}\rfloor$ and the upper bound for the $\ell_1$-norm of the hidden-to-output weights as $V=1000$. Totally, 3e4 iterations were performed. At the $k$th iteration, the learning rate is chosen to be $0.8/\log(e+k)$. Table \ref{Tab: significant Genes} summarizes the top 10 significant genes detected by SQLR and $F$-test. Based on the result, the top 10 genes having the smallest $P$-values detected by $F$-test and SQLR are similar.  
	
	\begin{table}[htbp]
		\centering
		\caption{Top 10 significant genes detected by the neural-network-based SQLR test and the linear- regression-based $F$-test}\label{Tab: significant Genes}
		\begin{tabular}{l|c|l|c}
			\hline
			\multicolumn{2}{c}{$F$-test} & \multicolumn{2}{c}{SQLR}\\
			\hline
			Gene & $P$-value	&  Gene &  $P$-value\\
			\hline
			\textit{SNRNP40}	& 5.48E-05	& \textit{PPIH}	& 5.84E-05\\
			\textit{PPIH}	& 1.01E-04	& \textit{SNRNP40}	&   6.91E-05\\
			\textit{GPR85}	& 1.65E-04	& \textit{NOD2}	&	1.22E-04\\
			\textit{DNAJB1}	& 1.87E-04	& \textit{DNAJB1}	&	1.66E-04\\
			\textit{WDR70}	&	1.91E-04	& \textit{CTBP1-AS2}	&	1.94E-04\\
			\textit{CYP4F2}	&	2.64E-04 & \textit{GPR85}	&	2.21E-04\\
			\textit{NOD2}	&	2.84E-04 & \textit{WDR70}	&	2.31E-04\\
			\textit{MEGF9}	&   2.85E-04	& \textit{KAZALD1}	&	2.59E-04\\
			\textit{CTBP1-AS2}	&	3.35E-04	& \textit{CYP4F2} & 2.95E-04\\
			\textit{HNRNPAB}	&   3.58E-04	& \textit{HNRNPAB}	& 3.72E-04\\
			\hline 			
		\end{tabular}
	\end{table}

	To explore the performance of the proposed SQLR test for categorical predictors, we conducted a genetic association analysis by applying SQLR to the ADNI genotype data in the \textit{APOE} gene. The \textit{APOE} gene in chromosome 19 is a well-known AD gene \citep{strittmatter1993apolipoprotein}. For this analysis, we considered all available single-nucleotide polymorphisms (SNPs) in the \textit{APOE} gene as the input feature and conducted single-locus association tests considering all other SNPs in the gene. We used the same response variable as the one used in the previous gene expression study. A total of 780 subjects and 169 SNPs were obtained after quality control.
	
	Same to the gene expression study, we used the sample average of the response variable as the null estimator. The tuning parameters used to fit neural networks are the same as mentioned before. Table \ref{Tab: significant SNPs} summarizes the top 10 significant SNPs in the \textit{APOE} gene detected by the SQLR method for neural networks and by the $F$-test in linear regression along with their $P$-values.
	
	As we can see from the result, the majority of significant SNPs detected by the $F$-test and SQLR test overlap. Whether these significant SNPs are biologically meaningful needs further investigation. This shows that the SQLR test based on neural networks has the potential for wider applications, at least in this study, it performs as good as the F-test.
	
	\begin{table}[htbp]
		\centering
		\caption{Top 10 significant SNPs detected by the SQLR for neural networks and the $F$-test in linear regression}\label{Tab: significant SNPs}
		\begin{tabular}{l|c|l|c}
			\hline
			\multicolumn{2}{c}{$F$-test} & \multicolumn{2}{c}{SQLR-neural net}\\
			\hline
			SNP & $P$-value	&  SNP &  $P$-value\\
			\hline
			rs10414043	& 1.10E-05	& rs10414043	& 1.18E-05\\
			rs7256200	& 1.10E-05	& rs7256200	&   1.18E-05\\
			rs769449	& 1.88E-05	& rs769449	&	2.00E-05\\
			rs438811	& 1.94E-05	& rs438811	&	2.28E-05\\
			rs10119	&	2.42E-05	& rs10119	&	2.59E-05\\
			rs483082	&	2.50E-05 & rs483082	&	2.91E-05\\
			rs75627662	&	5.32E-04 & rs75627662	&	5.44E-04\\
			rs\_x139	&   1.76E-03	& rs1038025	&	3.42E-03\\
			rs59325138	&	3.01E-03	& rs59325138 & 3.67E-03\\
			rs1038025	&   3.15E-03	& rs1038026	& 4.34E-03\\
			\hline 			
		\end{tabular}
	\end{table}

	\section{Discussion}
	Hypothesis-driven studies are quite common in biomedical and public health research. For instance, investigators are typically interested in detecting complex relationships (e.g., non-linear relationships) between genetic variants and diseases in genetic studies. Therefore, significance tests based on a flexible and powerful model is crucial in real world applications. Although neural networks have achieved great success in pattern recognition, due to its black-box nature, it is not straightforward to conduct statistical inference based on neural networks. To fill this gap, we proposed a sieve quasi-likelihood ratio test based on neural networks to testing complex associations. The asymptotic chi-squared distribution of the test statistic was developed, which was validated via simulations studies. We also evaluated SQLR by applying it to the gene expression and sequence data from ADNI.
	
	There are some limitations of the proposed method. First, the underlying function is required to be sufficiently smooth, which may not be true in some applications. Such requirement is not needed in the goodness of fit test proposed in \citet{shen2021goodness}. However, the construction of the goodness of fit test requires data splitting, which could potentially reduce its power. Our empirical studies also find that a suitable choice of the step size is crucial for decent performance of the proposed method. Further studies will be conducted on how to choose the suitable step size for our method so that it can be used as a guidance for real data applications. 
	
	In section 2, we developed general theories for the SQLR test under the framework of nonparametric regression. The conditions (C1)-(C3) are easy to verify compared with the original ones in \citet{shen2005sieve}. Such type of results can be extended to deep neural networks and other models used in artificial intelligence, such as convolution neural networks or long-short term memory reccurrent neural networks, as long as one can obtain a good bound on the metric entropy for the class of functions. 
	
	\section*{Acknowledgements}
	This work is supported by the National Institute on Drug Abuse (Award No. R01DA043501) and the National Library of Medicine (Award No. R01LM012848).

	\bibliography{NN-SQLR}
	
	\newpage
	\section*{Supplementary Materials}
	\subsection*{Proof of Theorem \ref{Thm: main result}}
	In this section, we take the sequence $\delta_n=o(n^{-1/2})$. The proof of the theorem relies on the following lemmas.
	\begin{lemma}\label{Lm: NormBound}
		Under (C1)-(C3), for a sufficiently large $n$,
		$$
		\norm{\pi_n\tilde{f}_n(\hat{f}_n)-f_0}_n^2-\norm{\hat{f}_n-f_0}_n^2\leq 2(1-\delta_n)\inprod{\hat{f}_n-f_0}{\delta_nu^*}_n+\mcal{O}_p(\delta_n^2).
		$$
	\end{lemma}
	
	\begin{proof}
		We first note that
		\begin{align*}
			\norm{\pi_n\tilde{f}_n(\hat{f}_n)-f_0}_n^2	& =\norm{\pi_n\tilde{f}_n(\hat{f}_n)-\tilde{f}_n(\hat{f}_n)+\tilde{f}_n(\hat{f}_n)-f_0}_n^2\\
			& =\norm{\pi_n\tilde{f}_n(\hat{f}_n)-\tilde{f}_n(\hat{f}_n)+(1-\delta_n)(\hat{f}_n-f_0)+\delta_nu^*}_n^2\\
			& =\norm{\pi_n\tilde{f}_n(\hat{f}_n)-\tilde{f}_n(\hat{f}_n)}_n^2+(1-\delta_n)^2\norm{\hat{f}_n-f_0}_n^2+\delta_n^2\norm{u^*}_n^2\\
			&	\qquad+2(1-\delta_n)\inprod{\pi_n\tilde{f}_n(\hat{f}_n)-\tilde{f}_n(\hat{f}_n)}{\hat{f}_n-f_0}_n+2\delta_n\inprod{\pi_n\tilde{f}_n(\hat{f}_n)-\tilde{f}_n(\hat{f}_n)}{u^*}_n\\
			& \qquad +2\delta_n(1-\delta_n)\inprod{\hat{f}_n-f_0}{u^*}_n\\
			& \leq (1-\delta_n)^2\norm{\hat{f}_n-f_0}_n^2+2(1-\delta_n)\inprod{\hat{f}_n-f_0}{\delta_nu^*}_n+\delta_n^2\norm{u^*}_n^2\\
			& \qquad+2(1-\delta_n)\norm{\pi_n\tilde{f}_n(\hat{f}_n)-\tilde{f}_n(\hat{f}_n)}_n\norm{\hat{f}_n-f_0}_n\\
			& \qquad+2\delta_n\norm{\pi_n\tilde{f}_n(\hat{f}_n)-\tilde{f}_n(\hat{f}_n)}_n\norm{u^*}_n+\norm{\pi_n\tilde{f}_n(\hat{f}_n)-\tilde{f}_n(\hat{f}_n)}_n^2.
		\end{align*}
		For a enough large $n$, the Strong Law of Large Numbers implies that $\norm{u^*}_n\leq2\norm{u^*}$ a.s. and hence
		$$
		\delta_n^2\norm{u^*}_n^2=\mcal{O}_p(\delta_n^2).
		$$
		Moreover,
		\begin{align*}
			(1-\delta_n)^2\norm{\hat{f}_n-f_0}_n^2-\norm{\hat{f}_n-f_0}_n^2	& =(-2\delta_n+\delta_n^2)\norm{\hat{f}_n-f_0}_n^2\\
			& \leq\delta_n^2\norm{\hat{f}_n-f_0}_n^2.
		\end{align*}
		On the other hand, under (C1) and (C2), it follows from Lemma 5.4 in \citet{geer2000empirical} that
		\begin{equation}\label{Eq: RoC-sample L2}
			\mbb{P}\left(\sup_{\substack{f\in\mcal{F}_n \\ \|f-f_0\|\leq\rho_n}}\norm{f-f_0}_n>8\rho_n\right)\leq4\exp\left(-n\rho_n^2\right),
		\end{equation}
		which implies that $\norm{\hat{f}_n-f_0}_n=\mcal{O}_p(\rho_n)=o_p(1)$ and then
		$$
		\delta_n^2\norm{\hat{f}_n-f_0}_n^2=o_p(\delta_n^2).
		$$
		Under (C2) and (C3), we have 
		\begin{align*}
			2(1-\delta_n)\norm{\pi_n\tilde{f}_n(\hat{f}_n)-\tilde{f}_n(\hat{f}_n)}_n\norm{\hat{f}_n-f_0}	& \leq2\norm{\hat{f}_n-f_0}_n\norm{\pi_n\tilde{f}_n(\hat{f}_n)-\tilde{f}_n(\hat{f}_n)}_n\\
			& =\mcal{O}_p(\rho_n)\cdot o_p(\rho_n^{-1}\delta_n^2)\\
			& =o_p(\delta_n^2),
		\end{align*}
		and for a large enough $n$,
		$$
		2\delta_n\norm{\pi_n\tilde{f}_n(\hat{f}_n)-\tilde{f}_n(\hat{f}_n)}_n\norm{u^*}_n=o_p(\delta_n\rho_n^{-1}\delta_n^2)=o_p(\delta_n^2),
		$$
		and
		$$
		\norm{\pi_n\tilde{f}_n(\hat{f}_n)-\tilde{f}_n(\hat{f}_n)}_n^2=o_p(\rho_n^{-2}\delta_n^2\delta_n^2)=o_p(\delta_n^2).
		$$
		Therefore, we obtain
		$$
		\norm{\pi_n\tilde{f}_n(\hat{f}_n)-f_0}_n^2-\norm{\hat{f}_n-f_0}_n^2\leq 2(1-\delta_n)\inprod{\hat{f}_n-f_0}{\delta_nu^*}_n+\mcal{O}_p(\delta_n^2).
		$$
	\end{proof}

	\begin{lemma}\label{Lm: MultiplierProcess}
		Under (C1) - (C3),
		$$
		n^{-1}\sum_{i=1}^n\epsilon_i\left(\pi_n\tilde{f}_n(\hat{f}_n)(\mbf{X}_i)-\hat{f}_n(\mbf{X}_i)\right)=n^{-1}\delta_n\sum_{i=1}^n\epsilon_iu^*(\mbf{X}_i)+o_p(\delta_n n^{-1/2}).
		$$
	\end{lemma}
	
	\begin{proof}
		Note that
		\begin{align*}
			&	n^{-1}\sum_{i=1}^n\epsilon_i\left(\pi_n\tilde{f}_n(\hat{f}_n)(\mbf{X}_i)-\hat{f}_n(\mbf{X}_i)\right) \\
			=	& n^{-1}\sum_{i=1}^n\epsilon_i\left(\pi_n\tilde{f}_n(\hat{f}_n)(\mbf{X}_i)-\tilde{f}_n(\hat{f}_n)(\mbf{X}_i)+\tilde{f}_n(\hat{f}_n)(\mbf{X}_i)-\hat{f}_n(\mbf{X}_i)\right)\\
			=	& n^{-1}\sum_{i=1}^n\epsilon_i\left(\pi_n\tilde{f}_n(\hat{f}_n)(\mbf{X}_i)-\tilde{f}_n(\hat{f}_n)(\mbf{X}_i)\right)-n^{-1}\delta_n\sum_{i=1}^n\epsilon_i\left(\hat{f}_n(\mbf{X}_i)-f_0(\mbf{X}_i)\right)+n^{-1}\delta_n\sum_{i=1}^n\epsilon_iu^*(\mbf{X}_i).\numberthis\label{Eq: decomposition}
		\end{align*}
		Now, we define
		$$
		J(\delta)=\int_0^\delta H^{1/2}(u)\mrm{d}u\vee\delta,
		$$
		and 
		$$
		\Delta_n=\sup_{\substack{f\in\mcal{F}_n \\ \norm{f-f_0}\leq\rho_n}}\norm{f-f_0}_n.
		$$
		Let $\xi_1,\ldots,\xi_n$ be i.i.d. Rademacher random variables independent of $\epsilon_1,\epsilon_n$ and $\mbf{X}_1,\ldots,\mbf{X}_n$. It then follows from Corollary 2.2.8 in \citet{van1996weak} that
		\begin{align*}
			\mbb{E}\left[\sup_{\substack{f\in\mcal{F}_n \\ \norm{f-f_0}\leq\rho_n}}\left|\frac{1}{\sqrt{n}}\sum_{i=1}^n\xi_i(f-f_0)(\mbf{X}_i)\right|\right]	&	= \mbb{E}\left[\mbb{E}\left[\left.\sup_{\substack{f\in\mcal{F}_n \\ \norm{f-f_0}\leq\rho_n}}\left|\frac{1}{\sqrt{n}}\sum_{i=1}^n\xi_i(f-f_0)(\mbf{X}_i)\right|\right|\Delta_n\leq 8\rho_n\right]\right]\\
			& \lesssim\mbb{E}\left[\int_0^{8\rho_n}\sqrt{\log N(u,\mcal{F}_n,L_2(\mbb{P}_n))}\mrm{d}u\right]\\
			& \lesssim \int_0^{8\rho_n}H^{1/2}(u)\mrm{d}u\\
			& \lesssim J(8\rho_n).
		\end{align*}
		Moreover, based on (C2), we obtain
		\begin{align*}
			\frac{J^2(8\rho_n)}{n\rho_n^4}	& =\frac{1}{n\rho_n^4}\left[\left(\int_0^{8\rho_n}H^{1/2}(u)\mrm{d}u\right)^2+64\rho_n^2\right]\\
			& =\frac{1}{n\rho_n^2}\left[64\left(\frac{1}{8\rho_n}\int_0^{8\rho_n}H^{1/2}(u)\mrm{d}u\right)^2+64\right]\\
			& =\frac{1}{n\rho_n^2}(64H(\lambda\rho_n)+64)\mrm{ for some }\lambda\in(0,8),
		\end{align*}
		where the last inequality follows from the mean value theorem for integrals. Hence, $J(8\rho_n)=\mcal{O}(\sqrt{n}\rho_n^2)$ and
		$$
		\mbb{E}^*\left[\sup_{\substack{f\in\mcal{F}_n \\ \norm{f-f_0}\leq\rho_n}}\left|\sum_{i=1}^n\xi_i(f-f_0)(\mbf{X}_i)\right|\right]\lesssim n\rho_n^2.
		$$
		By Proposition \ref{Prop: Rate of Convergence of MP}, we know that 
		$$
		\sup_{\substack{f\in\mcal{F}_n \\ \norm{f-f_0}\leq\rho_n}}\left|\sum_{i=1}^n\epsilon_i(f-f_0)(\mbf{X}_i)\right|=\mcal{O}_p\left(n\rho_n^2\right),
		$$
		which implies that
		\begin{equation}\label{Eq: UB for II}
			n^{-1}\delta_n\sum_{i=1}^n\epsilon_i\left(\hat{f}_n(\mbf{X}_i)-f_0(\mbf{X}_i)\right)=\mcal{O}_p(\delta_n\rho_n^2)=o_p(\delta_n n^{-1/2}).
		\end{equation}
		The desired claim then follows by combining (C3) with equations (\ref{Eq: UB for II}) and (\ref{Eq: decomposition}).
		%	On the other hand, we note that
		%	\begin{align*}
		%		n^{-1}\sum_{i=1}^n\epsilon_i\left(\pi_n\tilde{f}_n(\hat{f}_n)(X_i)-\tilde{f}_n(\hat{f}_n)(X_i)\right)	& =n^{-1}\sum_{i=1}^n\epsilon_i\left(\pi_n\tilde{f}_n(\hat{f}_n)(X_i)-f_0(X_i)\right)\\
		%			& \qquad -n^{-1}\sum_{i=1}^n\epsilon_i\left((1-\delta_n)\hat{f}_n(X_i)-f_0(X_i)\right)\\
		%			& \qquad+n^{-1}\delta_n\sum_{i=1}^n\epsilon_i(f_0(X_i)+u^*(X_i)).
		%	\end{align*}
		%	Under (C1), we know that $(1-\delta_n)\hat{f}_n\in\mcal{F}_n$. Moreover, under (C3)
		%	\begin{align*}
		%		\norm{\pi_n\tilde{f}_n(\hat{f}_n)-f_0}	& \leq\norm{\pi_n\tilde{f}_n(\hat{f}_n)-\tilde{f}_n(\hat{f}_n)}+\norm{\tilde{f}_n(\hat{f}_n)-f_0}\\
		%			& =\mcal{O}_p(\rho_n^{-1}\delta_n^2)+(1-\delta_n)\norm{\hat{f}_n-f_0}+\delta_n\norm{u^*}\\
		%			& \leq\mcal{O}_p(\rho_n^{-1}\delta_n^2)+\rho_n+\mcal{O}(\delta_n)\\
		%			& \lesssim\rho_n.
		%	\end{align*}
		%	Therefore,
		%	\begin{align*}
		%		\sum_{i=1}^n\epsilon_i\left(\pi_n\tilde{f}_n(\hat{f}_n)(X_i)-\tilde{f}_n(\hat{f}_n)(X_i)\right)	& \leq\sup_{\substack{f\in\mcal{F}_n \\ \norm{f-f_0}\leq\rho_n}}\abs{\sum_{i=1}^n\epsilon_i\left(f(X_i)-f_0(X_i)\right)}=\mcal{O}_p(n\rho_n^2)\\
		%		\sum_{i=1}^n\epsilon_i\left((1-\delta_n)\hat{f}_n(X_i)-f_0(X_i)\right)	& \leq\sup_{\substack{f\in\mcal{F}_n \\ \norm{f-f_0}\leq\rho_n}}\abs{\sum_{i=1}^n\epsilon_i\left(f(X_i)-f_0(X_i)\right)}=\mcal{O}_p(n\rho_n^2).
		%	\end{align*}
		%	Since $\norm{f_0+u^*}<\infty$ and $X$ and $\epsilon$ are independent, it follows from the central limit theorem that
		%	$$
		%	n^{-1}\delta_n\sum_{i=1}^n\epsilon_i(f_0(X_i)+u^*(X_i))=\mcal{O}_p(\delta_nn^{-1/2}).
		%	$$
	\end{proof}

	We now prove Theorem \ref{Thm: main result}.
	\begin{proof}
		Note that for $\|f-f_0\|\leq\rho_n$, we have
		\begin{align*}
			\norm{\tilde{f}_n(f)-f_0}	& =\norm{(1-\delta_n)f+\delta_n(f_0+u^*)-f_0}\\
			& =\norm{(1-\delta_n)(f-f_0)+\delta_nu^*}\\
			& \leq(1-\delta_n)\norm{f-f_0}+\delta_n\norm{u^*}.
		\end{align*}
		With probability tending to 1, $\norm{\tilde{f}_n(f)-f_0}\leq\rho_n$.
		Since
		\begin{align*}
			\mbb{Q}_n(f) & =n^{-1}\sum_{i=1}^n(Y_i-f(\mbf{X}_i))^2\\
			& =n^{-1}\sum_{i=1}^n(\epsilon_i+f_0(\mbf{X}_i)-f(\mbf{X}_i))^2\\
			& =n^{-1}\sum_{i=1}^n\epsilon_i^2-2n^{-1}\sum_{i=1}^n\epsilon_i(f(\mbf{X}_i)-f_0(\mbf{X}_i))+\norm{f-f_0}_n^2,
		\end{align*}
		we have
		\begin{align*}
			\mbb{Q}_n(\hat{f}_n)	& =n^{-1}\sum_{i=1}^n\epsilon_i^2-2n^{-1}\sum_{i=1}^n\epsilon_i(\hat{f}_n(\mbf{X}_i)-f_0(\mbf{X}_i))+\norm{\hat{f}_n-f_0}_n^2\\
			\mbb{Q}_n(\pi_n\tilde{f}_n(\hat{f}_n))	& =n^{-1}\sum_{i=1}^n\epsilon_i^2-2n^{-1}\sum_{i=1}^n\epsilon_i(\pi_n\tilde{f}_n(\hat{f}_n)(\mbf{X}_i)-f_0(\mbf{X}_i))+\norm{\pi_n\tilde{f}_n(\hat{f}_n)-f_0}_n^2.
		\end{align*}
		By subtracting these two equation, we have
		$$
		\mbb{Q}_n(\hat{f}_n)=\mbb{Q}_n(\pi_n\tilde{f}_n(\hat{f}_n))+2n^{-1}\sum_{i=1}^n\epsilon_i(\pi_n\tilde{f}_n(\hat{f}_n)(\mbf{X}_i)-\hat{f}_n(\mbf{X}_i))+\norm{\hat{f}_n-f_0}_n^2-\norm{\pi_n\tilde{f}_n(\hat{f}_n)-f_0}_n^2.
		$$
		It then follows from the definition of $\hat{f}_n$ that
		\begin{align*}
			-\mcal{O}_p(\delta_n^2)	& \leq\inf_{f\in\mcal{F}_n}\mbb{Q}_n(f)-\mbb{Q}_n(\hat{f}_n)\\
			& \leq\mbb{Q}_n(\pi_n\tilde{f}_n(\hat{f}_n))-\mbb{Q}_n(\hat{f}_n)\\
			& \leq\norm{\pi_n\tilde{f}_n(\hat{f}_n)-f_0}_n^2-\norm{\hat{f}_n-f_0}_n^2-2n^{-1}\sum_{i=1}^n\epsilon_i(\pi_n\tilde{f}_n(\hat{f}_n)(\mbf{X}_i)-\hat{f}_n(\mbf{X}_i))
		\end{align*}
		From Lemma \ref{Lm: NormBound}, we know that
		$$
		\norm{\pi_n\tilde{f}_n(\hat{f}_n)-f_0}_n^2-\norm{\hat{f}_n-f_0}_n^2\leq 2(1-\delta_n)\inprod{\hat{f}_n-f_0}{\delta_nu^*}_n+\mcal{O}_p(\delta_n^2).
		$$
		From Lemma \ref{Lm: MultiplierProcess}, we have
		$$
		n^{-1}\sum_{i=1}^n\epsilon_i\left(\pi_n\tilde{f}_n(\hat{f}_n)(\mbf{X}_i)-\hat{f}_n(\mbf{X}_i)\right)=n^{-1}\delta_n\sum_{i=1}^n\epsilon_iu^*(\mbf{X}_i)+o_p(\delta_n n^{-1/2}).
		$$
		Therefore, put all the pieces together, we have
		\begin{align*}
			-\mcal{O}_p(\delta_n^2)	& \leq 2(1-\delta_n)\inprod{\hat{f}_n-f_0}{\delta_nu^*}_n-2n^{-1}\delta_n\sum_{i=1}^n\epsilon_iu^*(\mbf{X}_i)+o_p(\delta_nn^{-1/2})\\
			& \leq 2\inprod{\hat{f}_n-f_0}{\delta_nu^*}_n-2n^{-1}\delta_n\sum_{i=1}^n\epsilon_iu^*(\mbf{X}_i)+o_p(\delta_nn^{-1/2}),
		\end{align*}
		which implies that
		$$
		-\inprod{\hat{f}_n-f_0}{u^*}_n+n^{-1}\sum_{i=1}^n\epsilon_iu^*(\mbf{X}_i)\leq \mcal{O}_p(\delta_n)+o_p(n^{-1/2})=o_p(n^{-1/2}).
		$$
		By replacing $u^*$ with $-u^*$, we get
		$$
		\abs{\inprod{\hat{f}_n-f_0}{u^*}_n-n^{-1}\sum_{i=1}^n\epsilon_iu^*(\mbf{X}_i)}\leq o_p(n^{-1/2}),
		$$
		and the desired result follows immediately.
	\end{proof}

	\newpage
	\subsection*{Proof of Theorem \ref{Thm: SQLR}}
	In what follows, we consider $\delta_n=-\inprod{\hat{f}_n-f_0}{v^*}$ and $\eta_n=o(\delta_n^2)$. Under (C1)-(C3), it follows from Theorem \ref{Thm: main result} and Lemma \ref{Lm: inprod} that if $\sup_{\mbf{x}\in\mcal{X}}|v^*(\mbf{x})|<\infty$
	\begin{align*}
		\delta_n	& =-\inprod{\hat{f}_n-f_0}{v^*}_n+o_p(n^{-1/2})\\
		& =-\frac{1}{n}\sum_{i=1}^n\epsilon_iv^*(\mbf{X}_i)+o_p(n^{-1/2})\\
		& =\mcal{O}_p(n^{-1/2}).
	\end{align*}
	The proof of Theorem \ref{Thm: SQLR} relies on the following lemmas.
	
	\begin{lemma}[Convergence Rate for $\hat{f}_n^0$]
		Under (C1) and (C2), 
		$$
		\norm{\hat{f}_n^0-f_0}=\mcal{O}_p(\rho_n).
		$$
	\end{lemma}
	
	\begin{proof}
		As $\norm{\pi_nf_0-f_0}\leq\norm{\hat{f}_n-f_0}=\mcal{O}_p(\rho_n)$, it suffices to show that $\norm{\hat{f}_n^0-\pi_nf_0}=\mcal{O}_p(\rho_n)$. Note that
		for any $M>0$,
		$$
		\sup_{\substack{f\in\mcal{F}_n^0 \\ \norm{f-f_0}> M\rho_n}}\mbb{M}_n(f)-\mbb{M}_n(\pi_nf_0)\leq\sup_{\substack{f\in\mcal{F}_n \\ \norm{f-f_0}> M\rho_n}}\mbb{M}_n(f)-\mbb{M}_n(\pi_nf_0).
		$$
		Under (C2), we know that $\eta_n=\mcal{O}(\rho_n^2)$. It then follows from the definition of $\hat{f}_n^0$ that, for any $\varepsilon>0$, there exists $K>0$ such that,
		$$
		\mbb{P}\left(\mbb{M}_n(\hat{f}_n^0)-\mbb{M}_n(\pi_nf_0)<-K\rho_n^2\right)<\epsilon,
		$$ 
		and hence
		\begin{align*}
			& \mbb{P}\left(\norm{\hat{f}_n^0-\pi_nf_0}>M\rho_n\right)\\ 
			\leq	& \mbb{P}\left(\norm{\hat{f}_n^0-\pi_nf_0}>M\rho_n,\mbb{M}_n(\hat{f}_n^0)\geq\mbb{M}_n(\pi_nf_0)-K\rho_n^2\right)+\mbb{P}\left(\mbb{M}_n(\hat{f}_n^0)-\mbb{M}_n(\pi_nf_0)<-K\rho_n^2\right)\\
			\leq	& \mbb{P}\left(\sup_{\substack{f\in\mcal{F}_n^0 \\ \norm{f-\pi_nf_0}> M\rho_n}}\mbb{M}_n(f)-\mbb{M}_n(\pi_nf_0)\geq-K\rho_n^2\right)+\varepsilon\\
			\leq 	& \mbb{P}\left(\sup_{\substack{f\in\mcal{F}_n \\ \norm{f-\pi_nf_0}> M\rho_n}}\mbb{M}_n(f)-\mbb{M}_n(\pi_nf_0)\geq-K\rho_n^2\right)+\varepsilon.
		\end{align*}
		Note that
		$$
		\mbb{M}_n(f)-\mbb{M}_n(\pi_nf_0)=\frac{2}{n}\sum_{i=1}^n\epsilon_i\left(f(\mbf{X}_i)-\pi_nf_0(\mbf{X}_i)\right)+\norm{\pi_nf_0-f_0}_n^2-\norm{f-f_0}_n^2.
		$$
		Let $\mcal{F}_{j,n}=\{f\in\mcal{F}_n:2^{j-1}M\rho_n\leq\norm{f-\pi_nf_0}<2^jM\rho_n\}$. By a standard argument of peeling technique, we have
		$$
		\mbb{P}\left(\sup_{\substack{f\in\mcal{F}_n \\ \norm{f-\pi_nf_0}> M\rho_n}}\mbb{M}_n(f)-\mbb{M}_n(\pi_nf_0)\geq-K\rho_n^2\right) \leq\sum_{j=1}^\infty\mbb{P}\left(\sup_{f\in\mcal{F}_{j,n}}\mbb{M}_n(f)-\mbb{M}_n(\pi_nf_0)\geq-K\rho_n^2\right)
		$$
		Let $M(f)=\mbb{E}[\mbb{M}_n(f)]=-\norm{f-f_0}^2$. Then on $\mcal{F}_{j,n}$
		$$
		M(f)-M(\pi_nf_0)\lesssim-2^{2j-2}M^2\rho_n^2-\norm{\pi_nf_0-f_0}^2,
		$$
		which implies that
		\begin{align*}
			&	\sup_{f\in\mcal{F}_{j,n}}\mbb{M}_n(f)-\mbb{M}_n(\pi_nf_0)\\
			\leq	& \sup_{f\in\mcal{F}_{j,n}}\mbb{M}_n(f)-\mbb{M}_n(\pi_nf_0)-(M(f)-M(\pi_nf_0))+\sup_{f\in\mcal{F}_{j,n}}M(f)-M(\pi_nf_0)\\
			\lesssim & \sup_{f\in\mcal{F}_{j,n}}\abs{\frac{2}{n}\sum_{i=1}^n\epsilon_i(f(\mbf{X}_i)-\pi_nf_0(\mbf{X}_i))}+\sup_{f\in\mcal{F}_{j,n}}\abs{\norm{f-f_0}_n^2-\norm{f-f_0}^2}\\
			& \qquad+\abs{\norm{\pi_nf_0-f_0}_n^2-\norm{\pi_nf_0-f_0}^2}-2^{2j-2}M^2\rho_n^2.
		\end{align*}
		Therefore
		\begin{align*}
			\mbb{P}\left(\sup_{f\in\mcal{F}_{j,n}}\mbb{M}_n(f)-\mbb{M}_n(\pi_nf_0)\geq -K\rho_n^2\right)\leq P_1+P_2+P_3,
		\end{align*}
		where
		\begin{align*}
			P_1	& :=\mbb{P}\left(\sup_{f\in\mcal{F}_{j,n}}\abs{\frac{1}{\sqrt{n}}\sum_{i=1}^n\epsilon_i(f(\mbf{X}_i)-\pi_nf_0(\mbf{X}_i))}\geq\left(2^{2j-5}M^2-\frac{K}{8}\right)\sqrt{n}\rho_n^2\right)\\
			P_2	& :=\mbb{P}\left(\sup_{f\in\mcal{F}_{j,n}}\sqrt{n}\abs{\norm{f-f_0}_n^2-\norm{f-f_0}^2}\geq\left(2^{2j-4}M^2-\frac{K}{4}\right)\sqrt{n}\rho_n^2\right)\\
			P_3	& :=\mbb{P}\left(\abs{\norm{f_0-\pi_nf_0}_n^2-\norm{f_0-\pi_nf_0}^2}\geq\left(2^{2j-3}M^2-\frac{K}{2}\right)\rho_n^2\right).
		\end{align*}
		As $\norm{f_0-\pi_n f_0}\lesssim\rho_n$, by Markov's inequality and triangle inequality,
		\begin{align*}
			P_3	& \leq\left[\left(2^{2j-3}M^2-\frac{K}{2}\right)\rho_n^2\right]^{-1}\mbb{E}\left[\abs{\norm{f_0-\pi_nf_0}_n^2-\norm{f_0-\pi_nf_0}^2}\right]\\
			& \lesssim\left[\left(2^{2j-3}M^2-\frac{K}{2}\right)\rho_n^2\right]^{-1}.
		\end{align*}
		For $P_2$, it follows from Chebyshev's inequality, symmetrization inequality, contraction principle and moment inequality (see Proposition 3.1 of \citet{gine2000exponential}) that
		\begin{align*}
			P_2	& \leq\left[\left(2^{2j-4}M^2-\frac{K}{4}\right)\sqrt{n}\rho_n^2\right]^{-2}\mbb{E}\left[\sup_{f\in\mcal{F}_{j,n}}n\abs{\norm{f-f_0}_n^2-\norm{f-f_0}^2}^2\right]\\
			& \lesssim\left[\left(2^{2j-4}M^2-\frac{K}{4}\right)\sqrt{n}\rho_n^2\right]^{-2}\mbb{E}\left[\sup_{f\in\mcal{F}_{j,n}}\abs{\frac{1}{\sqrt{n}}\sum_{i=1}^n\xi_i(f(\mbf{X}_i)-f_0(\mbf{X}_i))^2}^2\right]\\
			& \lesssim\left[\left(2^{2j-4}M^2-\frac{K}{4}\right)\sqrt{n}\rho_n^2\right]^{-2}\left\{\left(\mbb{E}\left[\sup_{f\in\mcal{F}_{j,n}}\abs{\frac{1}{\sqrt{n}}\xi_i(f(\mbf{X}_i)-f_0(\mbf{X}_i))}\right]\right)^2+2^{2j}M^2\rho_n^2+n^{-1}\right\}\\
			& \lesssim\left[\left(2^{2j-4}M^2-\frac{K}{4}\right)\sqrt{n}\rho_n^2\right]^{-2}\left\{2^{2j}M^2\rho_n^2+2^{2j}M^2\rho_n^2+n^{-1}\right\},
		\end{align*}
		where $\xi_1,\ldots,\xi_n$ are i.i.d. Rademacher random variables independent of $\mbf{X}_1,\ldots,\mbf{X}_n$. Similarly, we have
		\begin{align*}
			P_1	& \leq\left[\left(2^{2j-5}M^2-\frac{K}{8}\right)\sqrt{n}\rho_n^2\right]^{-2}\mbb{E}\left[\sup_{f\in\mcal{F}_{j,n}}\abs{\frac{1}{\sqrt{n}}\sum_{i=1}^n\epsilon_i\left(f(\mbf{X}_i)-\pi_nf_0(\mbf{X}_i)\right)}^2\right]\\
			& \lesssim \left[\left(2^{2j-5}M^2-\frac{K}{8}\right)\sqrt{n}\rho_n^2\right]^{-2}\left\{2^{2j}M^2\rho_n^2+\norm{\epsilon_1}_2^22^{2j}M^2\rho_n^2+n^{-1}\mbb{E}\left[\max_{1\leq i\leq n}|\epsilon_i|^2\right]\right\}\\
			& \lesssim\left[\left(2^{2j-5}M^2-\frac{K}{8}\right)\sqrt{n}\rho_n^2\right]^{-2}\left\{2^{2j}M^2\rho_n^2+\norm{\epsilon_1}_2^22^{2j}M^2\rho_n^2+\norm{\epsilon_1}_p^2n^{-1+2/p}\right\}.
		\end{align*}
		Then 
		\begin{align*}
			P_1+P_2	& \lesssim\left[\left(2^{2j-5}M^2-\frac{K}{8}\right)\sqrt{n}\rho_n^2\right]^{-2}\left\{2^{2j}M^2\rho_n^2+(1\vee\norm{\epsilon_1}_2)^22^{2j}M^2\rho_n^2+(1\vee\norm{\epsilon_1}_p)^2n^{-1+2/p}\right\}\\
			& \lesssim C_\epsilon\left[\left(\frac{2^jM}{2^{2j-5}M^2-\frac{K}{8}}\right)^2\frac{\rho_n^2}{n\rho_n^4}+\frac{1}{\left(2^{2j-5}M^2-\frac{K}{8}\right)^2n^{2-2/p}\rho_n^4}\right].
		\end{align*}
		Under (C2), we have $\rho_n\gtrsim n^{-\frac{1}{2}+\frac{1}{2p}}$, which implies that $n\rho_n^2\gtrsim n^{1/p}$. Moreover, for $M\geq\sqrt{2K}$, we have $2^{2j-5}M^2-\frac{K}{8}\geq 2^{2j-5}M^2-2^{2j-6}M^2=2^{2j-6}M^2$. Therefore, we obtain
		\begin{align*}
			\sum_{j=1}^\infty P_1+P_2+P_3	& \lesssim C_\epsilon\sum_{j=1}^\infty\left[\left(\frac{2^jM}{2^{2j-6}M^2}\right)^2\frac{1}{n\rho_n^2}+\frac{1}{\left(2^{2j-6}M^2\right)^2}+\frac{1}{2^{2j-4}M^2}\right]\\
			& \lesssim \frac{C_\epsilon}{M^2},
		\end{align*}
		which implies that, for a sufficiently large $M$, 
		$$
		\mbb{P}\left(\sup_{\substack{f\in\mcal{F}_n \\ \norm{f-\pi_nf_0}> M\rho_n}}\mbb{M}_n(f)-\mbb{M}_n(\pi_nf_0)\geq-K\rho_n^2\right)<\varepsilon.
		$$
		Therefore, 
		$$
		\sup_n\mbb{P}\left(\norm{\hat{f}_n^0-\pi_nf_0}>M\rho_n\right)<2\varepsilon,
		$$
		which implies $\norm{\hat{f}_n^0-\pi_nf_0}=\mcal{O}_p(\rho_n)$.
	\end{proof}
	
	\begin{lemma}[Local Approximation]
		Suppose that $u_n\rho_n^\omega=o(n^{-1/2})$. Then, under (C1)-(C3) and $H_0$,
		$$
		\mbb{Q}_n(\hat{f}_n)-\mbb{Q}_n(\hat{f}_n^0)=-\norm{\hat{f}_n-\tilde{f}_n(\hat{f}_n)}_n^2+o_p(n^{-1}),
		$$
		and
		$$
		\mbb{Q}_n(\pi_n\tilde{f}_n(\hat{f}_n))\leq\mbb{Q}_n(\hat{f}_n^0)+o_p(n^{-1}).
		$$
	\end{lemma}
	
	\begin{proof}
		First, note that for $u^*=v^*/\norm{v^*}^2$,
		\begin{align*}
			& \mbb{Q}_n(\hat{f}_n)-\mbb{Q}_n(\pi_n\tilde{f}_n(\hat{f}_n))\\
			=	& -2n^{-1}\sum_{i=1}^n\epsilon_i\left(\hat{f}_n(\mbf{X}_i)-\pi_n\tilde{f}_n(\hat{f}_n)(\mbf{X}_i)\right)-2\inprod{\hat{f}_n-f_0}{\pi_n\tilde{f}_n(\hat{f}_n)-\hat{f}_n}_n-\norm{\pi_n\tilde{f}_n(\hat{f}_n)-\hat{f}_n}_n^2\\
			=	& -2n^{-1}\sum_{i=1}^n\epsilon_i\left(\hat{f}_n(\mbf{X}_i)-\tilde{f}_n(\hat{f}_n)(\mbf{X}_i)\right)+2\inprod{\hat{f}_n-f_0}{\hat{f}_n-\tilde{f}_n(\hat{f}_n)}_n\\
			& -2n^{-1}\sum_{i=1}^n\epsilon_i\left(\tilde{f}_n(\hat{f}_n)(\mbf{X}_i)-\pi_n\tilde{f}_n(\hat{f}_n)(\mbf{X}_i)\right)+2\inprod{\hat{f}_n-f_0}{\tilde{f}_n(\hat{f}_n)-\pi_n\tilde{f}_n(\hat{f}_n)}_n-\norm{\pi_n\tilde{f}_n(\hat{f}_n)-\hat{f}_n}_n^2.
		\end{align*}
		Under (C3), it is clear that
		\begin{align*}
			2n^{-1}\sum_{i=1}^n\epsilon_i\left(\tilde{f}_n(\hat{f}_n)(\mbf{X}_i)-\pi_n\tilde{f}_n(\hat{f}_n)(\mbf{X}_i)\right)	& =o_p(\delta_n^2)\\
			2\inprod{\hat{f}_n-f_0}{\tilde{f}_n(\hat{f}_n)-\pi_n\tilde{f}_n(\hat{f}_n)}_n	& \leq 2\norm{\hat{f}_n-f_0}_n\norm{\tilde{f}_n(\hat{f}_n)-\pi_n\tilde{f}_n(\hat{f}_n)}_n\\
			& =\mcal{O}_p(\rho_n)\cdot o_p(\rho_n^{-1}\delta_n^2)=o_p(\delta_n^2).
		\end{align*}
		On the other hand, since $\hat{f}_n-\tilde{f}_n(\hat{f}_n)=-\delta_nv^*/\norm{v^*}^2$, we have
		\begin{align*}
			& \mbb{Q}_n(\hat{f}_n)-\mbb{Q}_n(\pi_n\tilde{f}_n(\hat{f}_n))\\
			=	& 2\delta_n\norm{v^*}^{-2}n^{-1}\sum_{i=1}^n\epsilon_iv^*(\mbf{X}_i)-2\delta_n\norm{v^*}^{-2}\inprod{\hat{f}_n-f_0}{v^*}_n-\norm{\pi_n\tilde{f}_n(\hat{f}_n)-\hat{f}_n}_n^2+o_p(\delta_n^2)\\
			=	& -\norm{\pi_n\tilde{f}_n(\hat{f}_n)-\hat{f}_n}_n^2+o_p(\delta_nn^{-1/2})+o_p(\delta_n^2),
		\end{align*}
		where the last equality follows from Theorem \ref{Thm: main result}. For a enough large $n$, we have 
		\begin{align*}
			\norm{\pi_n\tilde{f}_n(\hat{f}_n)-\hat{f}_n}_n^2	& =\norm{\pi_n\tilde{f}_n(\hat{f}_n)-\tilde{f}_n(\hat{f}_n)}_n^2+\norm{\tilde{f}_n(\hat{f}_n)-\hat{f}_n}_n^2+2\inprod{\pi_n\tilde{f}_n(\hat{f}_n)-\tilde{f}_n(\hat{f}_n)}{\tilde{f}_n(\hat{f}_n)-\hat{f}_n}_n\\
			& =o_p(\rho_n^{-2}\delta_n^{4})+\norm{\hat{f}_n-\tilde{f}_n(\hat{f}_n)}_n^2+o_p(\rho_n^{-1}\delta_n^3)\\
			& =\norm{\hat{f}_n-\tilde{f}_n(\hat{f}_n)}_n^2+o_p(\delta_n^2),
		\end{align*}
		which implies that
		\begin{equation}\label{Eq: AltLocalApprox}
			\mbb{Q}_n(\hat{f}_n)-\mbb{Q}_n(\pi_n\tilde{f}_n(\hat{f}_n))=-\norm{\hat{f}_n-\tilde{f}_n(\hat{f}_n)}_n^2+o_p(n^{-1}).
		\end{equation}
		By replacing $\hat{f}_n$ with $\hat{f}_n^0$ and considering $u^*=-v^*/\norm{v^*}^2$, we have
		\begin{align*}
			& \mbb{Q}_n(\tilde{f}_n^0)-\mbb{Q}_n(\pi_n\tilde{f}_n(\hat{f}_n^0))\\
			=	&	-2n^{-1}\sum_{i=1}^n\epsilon_i\left(\hat{f}_n^0(\mbf{X}_i)-\tilde{f}_n(\hat{f}_n^0)(\mbf{X}_i)\right)+2\inprod{\hat{f}_n^0-f_0}{\hat{f}_n^0-\pi_n\tilde{f}_n(\hat{f}_n^0)}_n\\
			& -2n^{-1}\sum_{i=1}^n\epsilon_i\left(\tilde{f}_n(\hat{f}_n^0)(\mbf{X}_i)-\pi_n\tilde{f}_n(\hat{f}_n^0)(\mbf{X}_i)\right)-\norm{\pi_n\tilde{f}_n(\hat{f}_n^0)-\hat{f}_n^0}_n^2.\\
			=	& -2\delta_n\norm{v^*}^{-2}n^{-1}\sum_{i=1}^n\epsilon_iv^*(\mbf{X}_i)+2\inprod{\hat{f}_n^0-\hat{f}_n}{\hat{f}_n^0-\tilde{f}_n(\hat{f}_n^0)}_n+2\inprod{\hat{f}_n-f_0}{\hat{f}_n^0-\tilde{f}_n(\hat{f}_n^0)}_n\\
			& +2\inprod{\hat{f}_n^0-f_0}{\tilde{f}_n(\hat{f}_n^0)-\pi_n\tilde{f}_n(\hat{f}_n^0)}_n-\norm{\pi_n\tilde{f}_n(\hat{f}_n^0)-\hat{f}_n^0}_n^2+o_p(\delta_n^2)\\
			= & -2\delta_n\norm{v^*}^{-2}\left[n^{-1}\sum_{i=1}^n\epsilon_iv^*(\mbf{X}_i)-\inprod{\hat{f}_n-f_0}{v^*}_n\right]+2\inprod{\hat{f}_n^0-\hat{f}_n}{\hat{f}_n^0-\tilde{f}_n(\hat{f}_n^0)}_n\\
			& -\norm{\pi_n\tilde{f}_n(\hat{f}_n^0)-\hat{f}_n^0}_n^2+o_p(\delta_n^2)\\
			=	& 2\inprod{\hat{f}_n^0-\hat{f}_n}{\hat{f}_n^0-\tilde{f}_n(\hat{f}_n^0)}_n-\norm{\pi_n\tilde{f}_n(\hat{f}_n^0)-\hat{f}_n^0}_n^2+o_p(n^{-1})\\
			=	& 2\inprod{\hat{f}_n^0-\hat{f}_n}{\hat{f}_n^0-\tilde{f}_n(\hat{f}_n^0)}_n-\norm{\hat{f}_n^0-\tilde{f}_n(\hat{f}_n^0)}_n^2+o_p(n^{-1}).
		\end{align*}
		For any $f\in\{f\in\mcal{F}_n^0:\norm{f-f_0}\leq\rho_n\}$, under $H_0$, we have
		\begin{align*}
			0	& =\phi(f)-\phi(f_0)=\phi_{f_0}'[f-f_0]+\mcal{O}(u_n\norm{f-f_0}^\omega)\\
			& =\inprod{f-f_0}{v^*}+o(n^{-1/2}),
		\end{align*}
		which implies that $\inprod{\hat{f}_n^0-\hat{f}_n}{v^*}_n=o_p(n^{-1/2})$. Moreover, note that $\hat{f}_n^0-\tilde{f}_n(\hat{f}_n^0)=\tilde{f}_n(\hat{f}_n)-\hat{f}_n$, we have
		\begin{align*}
			\inprod{\hat{f}_n^0-\hat{f}_n}{\hat{f}_n^0-\tilde{f}_n(\hat{f}_n^0)}_n	& =\inprod{\hat{f}_n^0-\hat{f}_n}{\tilde{f}_n(\hat{f}_n)-\hat{f}_n}_n\\
			& =\inprod{\hat{f}_n^0-\tilde{f}_n(\hat{f}_n)+\tilde{f}_n(\hat{f}_n)-\hat{f}_n}{\tilde{f}_n(\hat{f}_n)-\hat{f}_n}_n\\
			& =\inprod{\hat{f}_n^0-\tilde{f}_n(\hat{f}_n)}{\tilde{f}_n(\hat{f}_n)-\hat{f}_n}_n-\inprod{\tilde{f}_n(\hat{f}_n)-\hat{f}_n}{\tilde{f}_n(\hat{f}_n)-\hat{f}_n}_n\\
			& =\inprod{\hat{f}_n^0-\hat{f}_n-\delta_nv^*/\norm{v^*}^2}{\delta_nv^*/\norm{v^*}^2}_n+\norm{\hat{f}_n-\pi_n\tilde{f}_n(\hat{f}_n)}_n^2\\
			& =\delta_n\norm{v^*}^{-2}\inprod{\hat{f}_n^0-\hat{f}_n}{v^*}_n+\norm{\hat{f}_n-\pi_n\tilde{f}_n(\hat{f}_n)}_n^2-\delta_n^2\norm{v^*}_n^2/\norm{v^*}^4\\
			& =\delta_n\norm{v^*}^{-2}\inprod{\hat{f}_n^0-\hat{f}_n}{v^*}+\norm{\hat{f}_n-\pi_n\tilde{f}_n(\hat{f}_n)}_n^2\\
			&	\quad-\delta_n^2\norm{v^*}^{-2}+o_p(\delta_nn^{-1/2})+o_p(\delta_n^2),
		\end{align*}
		where the last equality follows from Lemma \ref{Lm: inprod} and $\norm{v^*}_n^2=\norm{v^*}^2+o_p(1)$ by the weak law of large numbers. Now, since $\delta_n=-\inprod{\hat{f}_n-f_0}{v^*}$, 
		\begin{align*}
			\delta_n\norm{v^*}^{-2}\inprod{\hat{f}_n^0-\hat{f}_n}{v^*}-\delta_n^2\norm{v^*}^{-2}	& = \delta_n\norm{v^*}^{-2}\left(\inprod{\hat{f}_n^0-\hat{f}_n}{v^*}-\delta_n\right)\\
			& =\delta_n\norm{v^*}^{-2}\left(\inprod{\hat{f}_n^0-\hat{f}_n}{v^*}+\inprod{\hat{f}_n-f_0}{v^*}\right)\\
			& =\delta_n\norm{v^*}^{-2}\inprod{\hat{f}_n^0-f_0}{v^*}\\
			& =o_p(\delta_nn^{-1/2}),
		\end{align*}
		and then
		$$
		\inprod{\hat{f}_n^0-\hat{f}_n}{\hat{f}_n^0-\tilde{f}_n(\hat{f}_n^0)}_n=\norm{\hat{f}_n-\pi_n\tilde{f}_n(\hat{f}_n)}_n^2+o_p(n^{-1}).
		$$
		Therefore,
		\begin{equation}\label{Eq: NullLocalApprox}
			\mbb{Q}_n(\tilde{f}_n^0)-\mbb{Q}_n(\pi_n\tilde{f}_n(\hat{f}_n^0))=\norm{\hat{f}_n-\pi_n\tilde{f}_n(\hat{f}_n)}_n^2+o_p(n^{-1}).
		\end{equation}
		From (\ref{Eq: AltLocalApprox}) and (\ref{Eq: NullLocalApprox}), we obtain
		\begin{align*}
			\mbb{Q}_n(\hat{f}_n)-\mbb{Q}_n(\hat{f}_n^0)	& \leq\inf_{f\in\mcal{F}_n}\mbb{Q}_n(f)-\mbb{Q}_n(\hat{f}_n^0)+o_p(n^{-1})\\
			& \leq\mbb{Q}_n(\pi_n\tilde{f}_n(\hat{f}_n^0))-\mbb{Q}_n(\hat{f}_n^0)+o_p(n^{-1})\\
			& =-\norm{\hat{f}_n-\tilde{f}_n(\hat{f}_n)}_n^2+o_p(n^{-1})\\
			& =\mbb{Q}_n(\hat{f}_n)-\mbb{Q}_n(\pi_n\tilde{f}_n(\hat{f}_n))+o_p(n^{-1}).
		\end{align*}
		which proves the desired results.
	\end{proof}
	
	The problem here is that $\pi_n\tilde{f}_n(\hat{f}_n)$ may not be in $\mcal{F}_n^0$, so we need to construct an approximate minimizer having similar properties. Set
	$$
	f_n^*[t]=\pi_n\tilde{f}_n(\hat{f}_n)+t\frac{v^*}{\norm{v^*}^2},\quad t\in\mbb{R}.
	$$
	Note that, for any $|t|\lesssim n^{-1/2}$ and a enough large $n$, 
	\begin{align*}
		\norm{\tilde{f}_n(\hat{f}_n)-f_0}	& \leq\norm{\hat{f}_n-f_0}+\abs{\delta_n}\norm{\frac{v^*}{\norm{v^*}^2}}\\
		& \leq\norm{\hat{f}_n-f_0}+\norm{\hat{f}_n-f_0}\norm{v^*}\frac{\norm{v^*}}{\norm{v^*}^2}\\
		& =\mcal{O}_p(\rho_n)\\
		\norm{f_n^*[t]-f_0}	& =\norm{\pi_n\tilde{f}_n(\hat{f}_n)-f_0}+\abs{t}\norm{\frac{v^*}{\norm{v^*}}}\\
		& =\mcal{O}_p(\rho_n)+\mcal{O}_p(n^{-1/2})\\
		& =\mcal{O}_p(\rho_n).
	\end{align*}
	Under $H_0$ and $u_n\rho_n^\omega=o(n^{-1/2})$, we have
	\begin{align*}
		\phi(\pi_nf_n^*[t])	& =\phi(\pi_nf_n^*[t])-\phi(f_0)\\	
		& =\inprod{\pi_nf_n^*[t]-f_0}{v^*}+r(t)\\
		& =\inprod{\pi_nf_n^*[t]-f_n^*[t]}{v^*}+\inprod{\pi_n\tilde{f}_n(\hat{f}_n)-f_0}{v^*}+t+r(t).	\numberthis\label{Eq: 17inNote}	
	\end{align*}
	By using the $C_r$-inequality, there exists a certain constant $C_\omega>0$ such that
	$$
	\norm{\pi_nf_n^*[t]-f_0}^\omega\leq C_\omega\left(\norm{\pi_nf_n^*[t]-f_n^*[t]}^\omega+\norm{\pi_nf_n^*-f_0}^\omega+|t|^\omega\right).
	$$
	Let $\Delta_n=\sup_{\substack{f\in\mcal{F}_n \\ \norm{f-f_0}\leq\rho_n}}\norm{\tilde{f}_n(f)-\pi_n\tilde{f}_n(f)}$, then
	\begin{align*}
		\tilde{t}	& =2\norm{v^*}\Delta_n+u_nC_\omega\left(\Delta_n^\omega+\norm{\pi_nf_n^*-f_0}^\omega+n^{-\omega/2}\norm{v^*}^\omega\right)\\
		& =\mcal{O}_p(\rho_n^{-1}\delta_n^2)+\mcal{O}_p(u_n(\rho_n^{-1}\delta_n^2)^\omega)+\mcal{O}_p(u_n\rho_n^\omega)+\mcal{O}_p(u_nn^{-\omega/2})\\
		& =o_p(n^{-1/2}).
	\end{align*}
	Therefore, by (\ref{Eq: 17inNote}),
	\begin{align*}
		\phi(\pi_nf_n^*[\tilde{t}])	& \geq\tilde{t}-\left(\abs{\inprod{\pi_nf_n^*[\tilde{t}]-f_n^*[\tilde{t}]}{v^*}}+\abs{\inprod{\pi_n\tilde{f}_n(\hat{f}_n)-f_0}{v^*}}+r(\tilde{t}))\right)\geq0\\
		\phi(\pi_nf_n^*[-\tilde{t}])	& \leq-\tilde{t}+\left(\abs{\inprod{\pi_nf_n^*[\tilde{t}]-f_n^*[\tilde{t}]}{v^*}}+\abs{\inprod{\pi_n\tilde{f}_n(\hat{f}_n)-f_0}{v^*}}+r(\tilde{t}))\right)\leq0.
	\end{align*}
	Furthermore, by continuity of $\phi(\pi_nf_n^*[t])$ and the mean value theorem, there exists some $t^*\in\mbb{R}$ such that $\phi(\pi_nf_n^*[t^*])=0$ and $|t^*|=o_p(n^{-1/2})$. This implies that $\pi_nf_n^*[t^*]\in\mcal{F}_n^0$. Clearly, $\norm{\pi_nf_n^*[t^*]-f_0}\leq\rho_n$ for a large $n$.

	\begin{lemma}
		Under (C1)-(C3), we have
		$$
		\mbb{Q}_n(\pi_nf_n^*)-\mbb{Q}_n(\pi_nf_n^*[t^*])=o_p(n^{-1}).
		$$
	\end{lemma}
	
	\begin{proof}
		Note that
		\begin{align*}
			& \mbb{Q}_n(\pi_n\tilde{f}_n(\hat{f}_n))-\mbb{Q}_n(\pi_nf_n^*[t^*])\\
			=	& -2n^{-1}\sum_{i=1}^n\epsilon_i\left(\pi_n\tilde{f}_n(\hat{f}_n)(\mbf{X}_i)-\pi_nf_n^*[t^*](\mbf{X}_i)\right)-\norm{\pi_nf_n^*[t^*]-f_0}_n^2+\norm{\pi_n\tilde{f}_n(\hat{f}_n)-f_0}_n^2\\
			=	& -2n^{-1}\sum_{i=1}^n\epsilon_i\left(\pi_n\tilde{f}_n(\hat{f}_n)(\mbf{X}_i)-\pi_nf_n^*[t^*](\mbf{X}_i)\right)-2\inprod{\pi_n\tilde{f}_n(\hat{f}_n)-f_0}{\pi_nf_n^*[t^*]-\tilde{f}_n(\hat{f}_n)}_n\\
			& -\norm{\pi_nf_n^*[t^*]-\pi_n\tilde{f}_n(\hat{f}_n)}_n^2.
		\end{align*}
		Moreover, since
		\begin{align*}
			\norm{\pi_nf_n^*[t^*]-\pi_n\tilde{f}_n(\hat{f}_n)}_n	& \leq\norm{\pi_nf_n^*[t^*]-f_n^*[t^*]}_n+\norm{f_n^*[t^*]-\pi_n\tilde{f}_n(\hat{f}_n)}_n\\
			& \leq\norm{\pi_nf_n^*[t^*]-f_n^*[t^*]}_n+|t^*|\frac{\norm{v^*}_n}{\norm{v^*}^2}\\
			& =\mcal{O}_p(\rho_n^{-1}\delta_n^2)+o_p(n^{-1/2})=o_p(n^{-1/2}),
		\end{align*}
		it follows from Cauchy-Schwarz inequality that
		\begin{align*}
			\inprod{\pi_n\tilde{f}_n(\hat{f}_n)-f_0}{\pi_nf_n^*[t^*]-\tilde{f}_n(\hat{f}_n)}_n	& =\inprod{\pi_n\tilde{f}_n(\hat{f}_n)-\tilde{f}_n(\hat{f}_n)}{\pi_nf_n^*[t^*]-\tilde{f}_n(\hat{f}_n)}_n\\
			& \quad+\inprod{\tilde{f}_n(\hat{f}_n)-f_0}{\pi_nf_n^*[t^*]-\tilde{f}_n(\hat{f}_n)}_n\\
			& \leq\norm{\pi_n\tilde{f}_n(\hat{f}_n)-\tilde{f}_n(\hat{f}_n)}_n\norm{\pi_nf_n^*[t^*]-\pi_n\tilde{f}_n(\hat{f}_n)}_n\\
			& \quad+\inprod{\tilde{f}_n(\hat{f}_n)-f_0}{\pi_nf_n^*[t^*]-\tilde{f}_n(\hat{f}_n)}_n\\
			& =o_p(n^{-1})+\inprod{\tilde{f}_n(\hat{f}_n)-f_0}{\pi_nf_n^*[t^*]-\tilde{f}_n(\hat{f}_n)}_n.
		\end{align*}
		On the other hand, note that
		\begin{align*}
			\inprod{\tilde{f}_n(\hat{f}_n)-f_0}{\pi_nf_n^*[t^*]-\pi_n\tilde{f}_n(\hat{f}_n)}_n	&  =\inprod{\tilde{f}_n(\hat{f}_n)-\hat{f}_n}{\pi_nf_n^*[t^*]-\pi_n\tilde{f}_n(\hat{f}_n)}_n+\inprod{\hat{f}_n-f_0}{\pi_nf_n^*[t^*]-f_n^*[t^*]}_n\\
			& \quad+\inprod{\hat{f}_n-f_0}{f_n^*[t^*]-\pi_n\tilde{f}_n(\hat{f}_n)}_n\\
			& \leq\norm{\tilde{f}_n(\hat{f}_n)-\hat{f}_n}_n\norm{\pi_nf_n^*[t^*]-\pi_n\tilde{f}_n(\hat{f}_n)}_n\\
			& \quad+\norm{\hat{f}_n-f_0}_n\norm{\pi_nf_n^*[t^*]-f_n^*[t^*]}_n+\inprod{\hat{f}_n-f_0}{f_n^*[t^*]-\pi_n\tilde{f}_n(\hat{f}_n)}_n\\
			& \leq\mcal{O}_p(\delta_n)o_p(n^{-1/2})+\mcal{O}_p(\delta_n^2)+\inprod{\hat{f}_n-f_0}{f_n^*[t^*]-\pi_n\tilde{f}_n(\hat{f}_n)}_n\\
			& =o_p(n^{-1})+\inprod{\hat{f}_n-f_0}{f_n^*[t^*]-\pi_n\tilde{f}_n(\hat{f}_n)}_n\\
			& =o_p(n^{-1})+t^*\inprod{\hat{f}_n-f_0}{\frac{v^*}{\norm{v^*}}}_n\\
			& =o_p(n^{-1})+o_p(n^{-1/2})\mcal{O}_p(n^{-1/2})=o_p(n^{-1}),
		\end{align*}
		which implies that
		\begin{align*}
			& \inprod{\tilde{f}_n(\hat{f}_n)-f_0}{\pi_nf_n^*[t^*]-\tilde{f}_n(\hat{f}_n)}_n\\ 
			=	&	\inprod{\tilde{f}_n(\hat{f}_n)-f_0}{\pi_nf_n^*[t^*]-\pi_n\tilde{f}_n(\hat{f}_n)}_n+\inprod{\tilde{f}_n(\hat{f}_n)-f_0}{\pi_n\tilde{f}_n(\hat{f}_n)-\tilde{f}_n(\hat{f}_n)}_n\\
			\leq & o_p(n^{-1})+\norm{\tilde{f}_n(\hat{f}_n)-f_0}_n\norm{\pi_n\tilde{f}_n(\hat{f}_n)-\tilde{f}_n(\hat{f}_n)}_n\\
			=	& o_p(n^{-1}).
		\end{align*}
		From (C3), we have
		\begin{align*}
			2n^{-1}\sum_{i=1}^n\epsilon_i\left(\pi_n\tilde{f}_n(\hat{f}_n)(\mbf{X}_i)-\pi_nf_n^*[t^*](\mbf{X}_i)\right)	& =2n^{-1}\sum_{i=1}^n\epsilon_i\left(\pi_n\tilde{f}_n(\hat{f}_n)(\mbf{X}_i)-f_n^*[t^*](\mbf{X}_i)\right)\\
			& \quad+2n^{-1}\sum_{i=1}^n\epsilon_i\left(f_n^*[t^*](\mbf{X}_i)-\pi_nf_n^*[t^*](\mbf{X}_i)\right)\\
			& =-2\norm{v^*}^{-2}n^{-1}t^*\sum_{i=1}^n\epsilon_iv^*(\mbf{X}_i)+o_p(n^{-1})\\
			& =o_p(n^{-1/2})\mcal{O}_p(n^{-1/2})+o_p(n^{-1})=o_p(n^{-1}),
		\end{align*}
		which proves the desired result.
	\end{proof}

	We now ready to finish the proof of Theorem \ref{Thm: SQLR}.
	
	\begin{proof}
		Note that
		\begin{align*}
			\frac{n}{\sigma^2}\norm{\hat{f}_n-\tilde{f}_n(\hat{f}_n)}_n^2	& =\frac{n}{\sigma^2}\abs{\inprod{\hat{f}_n-f_0}{v^*}}^2\frac{\norm{v^*}_n^2}{\norm{v^*}^4}\\
			& =\abs{n^{-1/2}\sigma^{-1}\norm{v^*}^{-1}\sum_{i=1}^n\epsilon_iv^*(\mbf{X}_i)+o_p(1)}^2\frac{\norm{v^*}_n^2}{\norm{v^*}^2}.
		\end{align*}
		It then follows from Theorem \ref{Thm: main result}, the smoothness assumption \ref{Eq: Smoothness of Functional}, the classical central limit theorem and the Slutsky Theorem that 
		\begin{align*}
			\frac{n}{\sigma^2}\norm{\hat{f}_n-\tilde{f}_n(\hat{f}_n)}_n^2\xrightarrow{d}\chi_1^2
		\end{align*}
		Based on the previous lemmas, we have
		\begin{align*}
			\mbb{Q}_n(\hat{f}_n^0)-\mbb{Q}_n(\hat{f}_n)	& \leq \mbb{Q}_n(\pi_nf_n^*[t^*])-\mbb{Q}_n(\hat{f}_n)+o_p(n^{-1})\\
			& =\mbb{Q}_n(\pi_n\tilde{f}_n(\hat{f}_n))-\mbb{Q}_n(\hat{f}_n)+o_p(n^{-1})\\
			& =\norm{\tilde{f}_n(\hat{f}_n)-\hat{f}_n}_n^2+o_p(n^{-1})\\
			\mbb{Q}_n(\hat{f}_n^0)-\mbb{Q}_n(\hat{f}_n)	& \geq\mbb{Q}_n(\hat{f}_n^0)-\mbb{Q}_n(\pi_n\tilde{f}_n(\hat{f}_n^0))-o_p(n^{-1})\\
			& =\norm{\tilde{f}_n(\hat{f}_n)-\hat{f}_n}_n^2+o_p(n^{-1}),
		\end{align*}
		so that the desired result follows.
	\end{proof}

	\newpage
	\subsection*{Rate of Convergence of Approximate Sieve Extremum Estimators}
	We start with a general result on the rate of convergence of sieve estimators under the setup of nonparametric regression. The notations in this section are inherited from section 2 in the main text.
	\begin{lemma}\label{Lm: RoC Lemma 1}
		For every $n$ and every $\delta>8\norm{f_0-\pi_nf_0}$, we have
		$$
		\sup_{\substack{f\in\mcal{F}_n \\ \delta/2<\norm{f-\pi_nf_0}\leq\delta}}\mbb{Q}(\pi_nf_0)-\mbb{Q}(f)\lesssim-\delta^2.
		$$
	\end{lemma}
	
	\begin{proof}
		First, note that
		\begin{align*}
			\mbb{Q}(\pi_nf_0)-\mbb{Q}(f)	& =\mbb{E}\left[(Y-\pi_nf_0(\mbf{X}))^2\right]-\mbb{E}\left[(Y-f(\mbf{X}))^2\right]\\
			& =\mbb{E}\left[\left(\pi_nf_0(\mbf{X})-f_0(\mbf{X})\right)^2\right]-\mbb{E}\left[(f(\mbf{X})-f_0(\mbf{X}))^2\right]\\
			& =\norm{\pi_nf_0-f_0}^2-\norm{f-f_0}^2.
		\end{align*}
		The triangle inequality gives
		\begin{align*}
			\norm{f-\pi_nf_0} & \leq\norm{f-f_0}+\norm{\pi_nf_0-f_0}\\
			& =\norm{f-f_0}-\norm{\pi_nf_0-f_0}+2\norm{\pi_nf_0-f_0}.
		\end{align*}
		Therefore, we have
		$$
		\norm{f-f_0}-\norm{\pi_nf_0-f_0}\geq\norm{f-\pi_nf_0}-2\norm{\pi_nf_0-f_0}
		$$
		so that for every $f$ satisfying $\norm{f-\pi_nf_0}^2\geq16\norm{\pi_nf_0-f_0}^2$, i.e. $\norm{f-\pi_nf_0}\geq4\norm{\pi_nf_0-f_0}$, we have
		\begin{align*}
			\norm{f-f_0}-\norm{\pi_nf_0-f_0}	& \geq\norm{f-\pi_nf_0}-\frac{1}{2}\norm{\pi_nf_0-f_0}\\
			& =\frac{1}{2}\norm{f-\pi_nf_0}\geq0,\numberthis\label{Eq: to be squared}
		\end{align*}
		which implies that $\norm{f-f_0}\geq\norm{\pi_nf_0-f_0}$. By squaring both sides of (\ref{Eq: to be squared}), we obtain
		\begin{align*}
			\frac{1}{4}\norm{f-\pi_nf_0}^2	& \leq\norm{f-f_0}^2+\norm{\pi_nf_0-f_0}^2-2\norm{f-f_0}\cdot\norm{\pi_nf_0-f_0}\\
			& \leq\norm{f-f_0}^2+\norm{\pi_nf_0-f_0}^2-2\norm{\pi_nf_0-f_0}^2\\
			& =\norm{f-f_0}^2-\norm{\pi_nf_0-f_0}^2.
		\end{align*}
		Hence, for $\delta>8\norm{\pi_nf_0-f_0}$, we have
		\begin{align*}
			& \sup_{\substack{f\in\mcal{F}_n \\ \delta/2<\norm{f-f_0}\leq\delta}}\mbb{Q}(\pi_nf_0)-\mbb{Q}(f)\\
			\leq & \sup_{\substack{f\in\mcal{F}_n \\ \norm{f-\pi_nf_0}>\delta/2}}\norm{\pi_nf_0-f_0}^2-\norm{f-f_0}^2\\
			\leq & \sup_{\substack{f\in\mcal{F}_n \\ \norm{f-\pi_nf_0}>\delta/2}}\left(-\frac{1}{4}\norm{f-\pi_nf_0}^2\right)\\
			& \lesssim-\delta^2.
		\end{align*}
	\end{proof}

	\begin{lemma}\label{Lm: RoC Lemma 2}
		For every sufficiently large $n$ and $\delta>8\norm{f_0-\pi_nf_0}$, under (C1)
		$$
		\mbb{E}\left[\sup_{\substack{f\in\mcal{F}_n \\ \delta/2<\norm{f-\pi_nf_0}\leq\delta}}\sqrt{n}\abs{(\mbb{Q}_n-Q)(f)-(\mbb{Q}_n-Q)(\pi_nf_0)}\right]\lesssim\int_0^\delta H^{1/2}(u)\mrm{d}u+1
		$$
	\end{lemma}
	
	\begin{proof}
		Note that
		\begin{align*}
			& \sqrt{n}\abs{(\mbb{Q}_n-Q)(\pi_nf_0)-(\mbb{Q}_n-Q)(f)}\\
			=	& \sqrt{n}\abs{\frac{1}{n}\sum_{i=1}^n(Y_i-\pi_nf_0(\mbf{X}_i))^2-\mbb{E}\left[(Y-\pi_nf_0(\mbf{X}))^2\right]-\frac{1}{n}\sum_{i=1}^n\left(Y_i-f(\mbf{X}_i)\right)^2+\mbb{E}\left[(Y-f(\mbf{X}))^2\right]}\\
			\leq & \abs{\frac{2}{\sqrt{n}}\sum_{i=1}^n\epsilon_i(f(\mbf{X}_i)-\pi_nf_0(\mbf{X}_i))}+\abs{\frac{1}{\sqrt{n}}\sum_{i=1}^n\left\{(f_0(\mbf{X}_i)-\pi_nf_0(\mbf{X}_i))^2-\mbb{E}\left[(f_0(\mbf{X})-\pi_nf_0(\mbf{X}))^2\right]\right\}}\\
			& \qquad+\abs{\frac{1}{\sqrt{n}}\sum_{i=1}^n\left\{(f(\mbf{X}_i)-f_0(\mbf{X}_i))^2-\mbb{E}\left[(f(\mbf{X})-f_0(\mbf{X}))^2\right]\right\}},
		\end{align*}
		we obtain
		\begin{align*}
			& \mbb{E}\left[\sup_{\substack{f\in\mcal{F}_n \\ \delta/2<\norm{f-\pi_nf_0}\leq\delta}}\sqrt{n}\abs{(\mbb{Q}_n-Q)(f)-(\mbb{Q}_n-Q)(\pi_nf_0)}\right]\\
			\leq & \mbb{E}\left[\sup_{\substack{f\in\mcal{F}_n \\ \delta/2<\norm{f-\pi_nf_0}\leq\delta}}\abs{\frac{2}{\sqrt{n}}\sum_{i=1}^n\epsilon_i(f(\mbf{X}_i)-\pi_nf_0(\mbf{X}_i))}\right]\\
			& \qquad+\mbb{E}\left[\sup_{\substack{f\in\mcal{F}_n \\ \delta/2<\norm{f-\pi_nf_0}\leq\delta}}\abs{\frac{1}{\sqrt{n}}\sum_{i=1}^n\left\{(f(\mbf{X}_i)-f_0(\mbf{X}_i))^2-\mbb{E}\left[(f(\mbf{X})-f_0(\mbf{X}))^2\right]\right\}}\right]\\
			& \qquad+\mbb{E}\left[\abs{\frac{1}{\sqrt{n}}\sum_{i=1}^n\left\{(f_0(\mbf{X}_i)-\pi_nf_0(\mbf{X}_i))^2-\mbb{E}\left[(f_0(\mbf{X})-\pi_nf_0(\mbf{X}))^2\right]\right\}}\right]\\
			:=	& P_1+P_2+P_3.
		\end{align*}
		We start by bounding $P_3$. As $0\leq(f_0(\mbf{X}_i)-\pi_nf_0(\mbf{X}_i))^2\leq\norm{f_0-\pi_nf_0}_\infty^2$, it follows from Hoeffding's inequality that
		\begin{align*}
			\mbb{P}\left(\abs{\frac{1}{\sqrt{n}}\sum_{i=1}^n\left\{(f_0(\mbf{X}_i)-\pi_nf_0(\mbf{X}_i))^2-\mbb{E}\left[(f_0(\mbf{X})-\pi_nf_0(\mbf{X}))^2\right]\right\}}\geq t\right)\leq2\exp\left\{-\frac{2t^2}{\norm{f_0-\pi_nf_0}_\infty^4}\right\},
		\end{align*}
		and hence
		\begin{align*}
			P_3= & \int_0^\infty\mbb{P}\left(\abs{\frac{1}{\sqrt{n}}\sum_{i=1}^n\left\{(f_0(\mbf{X}_i)-\pi_nf_0(\mbf{X}_i))^2-\mbb{E}\left[(f_0(\mbf{X})-\pi_nf_0(\mbf{X}))^2\right]\right\}}\geq t\right)\mrm{d}t\\
			\leq & \int_0^\infty 2\exp\left\{-\frac{2t^2}{\norm{f_0-\pi_nf_0}_\infty^4}\right\}\mrm{d}t\\
			=	& (\pi/2)^{1/2}\norm{f_0-\pi_nf_0}_\infty^2\to0,\quad\mrm{as }n\to\infty,
		\end{align*}
		which implies that for a sufficiently large $n$,
		$$
		P_3=\mbb{E}\left[\abs{\frac{1}{\sqrt{n}}\sum_{i=1}^n\left\{(f_0(\mbf{X}_i)-\pi_nf_0(\mbf{X}_i))^2-\mbb{E}\left[(f_0(\mbf{X})-\pi_nf_0(\mbf{X}))^2\right]\right\}}\right]\leq1.
		$$
		%	Next, for $\delta>8\norm{f_0-\pi_nf_0}$ and $f\in\mcal{F}_n$ with $\delta/2<\norm{f-\pi_nf_0}\leq\delta$, we have
		%	$$
		%	\norm{f-f_0}\leq\norm{f-\pi_nf_0}+\norm{\pi_nf_0-f_0}\leq\delta+\frac{\delta}{8}\leq2\delta.
		%	$$
		
		On the other hand, it follows from symmetrization inequality that
		\begin{align*}
			P_2 & \leq\mbb{E}\left[\sup_{\substack{f\in\mcal{F}_n \\ \norm{f-\pi_nf_0}\leq\delta}}\abs{\frac{1}{\sqrt{n}}\sum_{i=1}^n\left\{(f(\mbf{X}_i)-f_0(\mbf{X}_i))^2-\mbb{E}\left[(f(\mbf{X})-f_0(\mbf{X}))^2\right]\right\}}\right]\\
			& \lesssim\mbb{E}\left[\sup_{\substack{f\in\mcal{F}_n \\ \norm{f-\pi_nf_0}\leq\delta}}\abs{\frac{1}{\sqrt{n}}}\sum_{i=1}^n\xi_i(f(\mbf{X}_i)-f_0(\mbf{X}_i))^2\right],
		\end{align*}
		where $\xi_1,\ldots,\xi_n$ are i.i.d. Rademacher random variables independent of $\mbf{X}_1,\ldots,\mbf{X}_n$. On the other hand, since $\mcal{F}_n$ is uniformly bounded, we know that
		$$
		\norm{f-f_0}_\infty\leq\norm{f}_\infty+\norm{f_0}_\infty<\infty.
		$$
		According to the contraction principle and Corollary 2.2.8 in \citet{van1996weak},
		\begin{align*}
			& \mbb{E}\left[\sup_{\substack{f\in\mcal{F}_n \\ \norm{f-\pi_nf_0}\leq\delta}}\abs{\frac{1}{\sqrt{n}}}\sum_{i=1}^n\xi_i(f(\mbf{X}_i)-f_0(\mbf{X}_i))^2\right]\\
			\lesssim	& \mbb{E}\left[\sup_{\substack{f\in\mcal{F}_n \\ \norm{f-\pi_nf_0}\leq\delta}}\abs{\frac{1}{\sqrt{n}}\sum_{i=1}^n\xi_i(f(\mbf{X}_i)-f_0(\mbf{X}_i))}\right]\\
			=	& \mbb{E}\left[\left.\mbb{E}\left[\sup_{\substack{f\in\mcal{F}_n \\ \norm{f-\pi_nf_0}\leq\delta}}\left|\frac{1}{\sqrt{n}}\sum_{i=1}^n\xi_i\left(f(\mbf{X}_i)-f_0(\mbf{X}_i)\right)\right|\right]\right|\Delta_n\leq\delta\right]\\
			\lesssim	& \mbb{E}\left[\int_0^\delta\sqrt{\log N(u,\mcal{F}_n,L_2(\mbb{P}_n))}\mrm{d}u\right]+\mbb{E}\left[\abs{\frac{1}{\sqrt{n}}\sum_{i=1}^n\xi_i(\pi_nf_0(\mbf{X}_i)-f_0(\mbf{X}_i))}\right]\\
			\lesssim	& \int_0^\delta H^{1/2}(u)\mrm{d}u+\mbb{E}\left[\abs{\frac{1}{\sqrt{n}}\sum_{i=1}^n\xi_i(\pi_nf_0(\mbf{X}_i)-f_0(\mbf{X}_i))}\right],
		\end{align*}
		where $\Delta_n=\sup_{\substack{f\in\mcal{F}_n \\ \norm{f-\pi_nf_0}\leq\delta}}\norm{f-f_0}_n$. Similar to the arguments used in bounding $P_3$, we have
		\begin{align*}
			& \mbb{E}\left[\abs{\frac{1}{\sqrt{n}}\sum_{i=1}^n\xi_i(\pi_nf_0(\mbf{X}_i)-f_0(\mbf{X}_i))}\right]\\
			=	& \int_0^\infty\mbb{P}\left(\frac{1}{\sqrt{n}}\abs{\sum_{i=1}^n\xi_i(\pi_nf_0(\mbf{X}_i)-f_0(\mbf{X}_i))}\geq t\right)\mrm{d}t\\
			\lesssim & \int_0^\infty\exp\left\{-\frac{t^2}{2\norm{\pi_nf_0-f_0}_\infty^2}\right\}\mrm{d}t\\
			\lesssim & \norm{\pi_nf_0-f_0}_\infty\to0\mrm{ as }n\to\infty,
		\end{align*}
		which implies that $\mbb{E}\left[\abs{\frac{1}{\sqrt{n}}\sum_{i=1}^n\xi_i(\pi_nf_0(\mbf{X}_i)-f_0(\mbf{X}_i))}\right]\leq1$ for a sufficiently large $n$. Therefore, for a  sufficiently large $n$,
		$$
		P_2\lesssim\int_0^\delta H^{1/2}(u)\mrm{d}u+1.
		$$
		
		Finally, for every sufficiently large $n$, as $\norm{\epsilon}_{p,1}<\infty$ for some $p\geq2$, it follows from the multiplier inequality (Lemma 2.9.1 in \citet{van1996weak}) that
		\begin{align*}
			P_3 =	& \mbb{E}\left[\sup_{\substack{f\in\mcal{F}_n \\ \delta/2<\norm{f-\pi_nf_0}\leq\delta}}\abs{\frac{1}{\sqrt{n}}}\sum_{i=1}^n\epsilon_i(f(\mbf{X}_i)-f_0(\mbf{X}_i))\right]\\
			\lesssim	& \max_{1\leq k\leq n}\mbb{E}\left[\sup_{\substack{f\in\mcal{F}_n \\ \delta/2<\norm{f-\pi_nf_0}\leq\delta}}\abs{\frac{1}{\sqrt{k}}\sum_{i=1}^k\xi_i(f(\mbf{X}_i)-f_0(\mbf{X}_i))}\right]\\
			\lesssim	& \int_0^\delta H^{1/2}(u)\mrm{d}u+1,
		\end{align*}
		where the last inequality follows as the upper bound for the local Rademacher complexity does not depend on $k$. Combining all the pieces together, we obtain the desired result.
	\end{proof}
	
	Based on Lemma \ref{Lm: RoC Lemma 1} and Lemma \ref{Lm: RoC Lemma 2}, the rate of convergence for the approximate sieve extremum estimators can be easily obtained via an application of Theorem 3.4.1 in \citet{van1996weak}.
	\begin{theorem}\label{Thm: RoC}
		Suppose that $\int_0^\delta H^{1/2}(u)\mrm{d}u\lesssim\phi_n(\delta)$ for some function $\phi_n:(0,\infty)\to\mbb{R}$ and for every sufficiently large $n$ and $8\norm{f_0-\pi_nf_0}<\delta\leq\eta$. Suppose that $\delta^{-\alpha}\phi_n(\delta)$ is decreasing on $(8\norm{f_0-\pi_nf_0},\infty)$ for some $\alpha<2$. Let $\rho_n\gtrsim\norm{f_0-\pi_nf_0}$ satisfy
		$$
		\rho_n^{-2}\phi_n(\rho_n)\lesssim\sqrt{n}\mrm{ for every }n.
		$$
		Then if $\rho_n=\mcal{O}_p(\rho_n^2)$ and $\norm{\hat{f}_n-\pi_nf_0}=o_p(1)$, we get
		$$
		\norm{\hat{f}_n-\pi_nf_0}=\mcal{O}_p(\rho_n)
		$$
		and
		$$
		\norm{\hat{f}_n-\pi_nf_0}=\mcal{O}_p(\max\left\{\rho_n,\norm{f_0-\pi_nf_0}\right\}).
		$$
	\end{theorem}
	
	\newpage
	\subsection*{Rate of Convergence of Multiplier Processes}
	\begin{prop}[Proposition 5 in \citet{han2019convergence}]\label{Prop: 5HanWellner2019}
		Suppose that $\epsilon_1,\ldots,\epsilon_n$ are i.i.d. mean zero random variables independent of i.i.d. random variables $\mbf{X}_1,\ldots,\mbf{X}_n$. Then, for any function class $\mcal{F}$,
		$$
		\mbb{E}\left[\sup_{f\in\mcal{F}}\abs{\sum_{i=1}^n\epsilon_if(\mbf{X}_i)}\right]\leq\mbb{E}\left[\sum_{k=1}^n\left(|\eta_{(k)}|-|\eta_{(k+1)}|\right)\mbb{E}\left[\sup_{f\in\mcal{F}}\abs{\sum_{i=1}^n\xi_if(\mbf{X}_i)}\right]\right],
		$$
		where $\xi_1,\ldots,\xi_n$ are i.i.d. Rademacher random variables independent of $\mbf{X}_1,\ldots,\mbf{X}_n$ and $\epsilon_1,\ldots,\epsilon_n$ and $|\eta_{(1)}|\geq\cdots\geq|\eta_{(n)}|\geq|\eta_{(n+1)}|\geq0$ are the reversed order statistics for $\{|\epsilon_i-\epsilon_i'|\}_{i=1}^n$ with $\{\epsilon_i'\}$ being an indepedent copy of $\{\epsilon_i\}$.
	\end{prop}
	
	As a consequence of Proposition \ref{Prop: 5HanWellner2019}, we can obtain the following result.
	
	\begin{prop}\label{Prop: Rate of Convergence Expectation of MP}
		Under the same assumption in Proposition \ref{Prop: 5HanWellner2019} and 
		$$
		\mbb{E}\left[\sup_{f\in\mcal{F}}\abs{\sum_{i=1}^k\xi_if(\mbf{X}_i)}\right]\lesssim k\rho_k^2\quad\mrm{ for all }k=1,\ldots,n.
		$$
		\begin{enumerate}[(i)]
			\item If $\norm{\epsilon}_p<\infty$ for some $p\geq1$ and the sequence $\{k\rho_k^2\}$ is non-decreasing, then
			$$
			\mbb{E}\left[\sup_{f\in\mcal{F}}\abs{\sum_{i=1}^n\epsilon_if(\mbf{X}_i)}\right]\lesssim\norm{\epsilon}_pn^{1+\frac{1}{p}}\rho_n^2.
			$$
			
			\item If $\norm{\epsilon}_{p,1}<\infty$ for some $p\geq1$ and the sequence $\{k^{1-\frac{1}{p}}\rho_k^2\}$ is non-decreasing, then
			$$
			\mbb{E}\left[\sup_{f\in\mcal{F}}\abs{\sum_{i=1}^n\epsilon_if(\mbf{X}_i)}\right]\lesssim\norm{\epsilon}_{p,1}n\rho_n^2.
			$$
		\end{enumerate}
	\end{prop}
	
	\begin{proof}
		\begin{enumerate}[(i)]
			\item According to Proposition \ref{Prop: 5HanWellner2019}, we have
			\begin{align*}
				\mbb{E}\left[\sup_{f\in\mcal{F}}\abs{\sum_{i=1}^n\epsilon_if(\mbf{X}_i)}\right]	& \lesssim\mbb{E}\left[\sum_{k=1}^n\left(|\eta_{(k)}|-|\eta_{(k+1)}|\right)k\rho_k^2\right]\\
				& \lesssim\mbb{E}\left[n^2\sum_{k=1}^n\left(|\eta_{(k)}|-|\eta_{(k+1)}|\right)\right]\\
				& =n^2\mbb{E}\left[|\eta_{(1)}|\right]\\
				& =n^2\mbb{E}\left[\max_{1\leq i\leq n}|\epsilon_i-\epsilon_i'|\right]\\
				& \lesssim n\delta_n^2\mbb{E}\left[\max_{1\leq i\leq n}|\xi_i|\right].
			\end{align*}
			Since $\norm{\epsilon}_p<\infty$, then it follows that
			$$
			\mbb{E}\left[\max_{1\leq i\leq n}|\epsilon_i|\right]\leq n^{1/p}\max_{1\leq i\leq n}\norm{\epsilon_i}_p=\norm{\epsilon}_p\cdot n^{1/p}.
			$$
			Therefore, we get
			$$
			\mbb{E}\left[\sup_{f\in\mcal{F}}\abs{\sum_{i=1}^n\epsilon_if(\mbf{X}_i)}\right]\lesssim\norm{\epsilon}_pn^{1+\frac{1}{p}}\delta_n^2.
			$$
			
			\item According to Proposition \ref{Prop: 5HanWellner2019}, we have
			\begin{align*}
				\mbb{E}\left[\sup_{f\in\mcal{F}}\abs{\sum_{i=1}^n\epsilon_if(\mbf{X}_i)}\right]	& \leq\mbb{E}\left[\sum_{k=1}^nk^{1/p}\left(|\eta_{(k)}|-|\eta_{(k+1)}|\right)\mbb{E}\left[\sup_{f\in\mcal{F}}\abs{k^{-1/p}\sum_{i=1}^n\xi_if(\mbf{X}_i)}\right]\right]\\
				& \leq\mbb{E}\left[\sum_{k=1}^nk^{1/p}\left(|\eta_{(k)}|-|\eta_{(k+1)}|\right)\right]\cdot\max_{1\leq k\leq n}\mbb{E}\left[\sup_{f\in\mcal{F}}\abs{k^{-1/p}\sum_{i=1}^n\xi_if(\mbf{X}_i)}\right].
			\end{align*}
			For $|\eta_{(k+1)}|<t\leq|\eta_{(k+1)}|$, we have $k=\mrm{Card}(\{i:|\epsilon_i-\epsilon_i'|\geq t\})$, and then
			\begin{align*}
				\mbb{E}\left[\sum_{k=1}^nk^{1/p}\left(|\eta_{(k)}|-|\eta_{(k+1)}|\right)\right]	& =\mbb{E}\left[\sum_{k=1}^n\int_{|\eta_{(k+1)}|}^{|\eta_{(k)}|}k^{1/p}\mrm{d}t\right]\\
				& =\mbb{E}\left[\int_0^{|\eta_{(1)}|}\left(\mrm{Card}(\{i:|\epsilon_i-\epsilon_i'|\geq t\})\right)^{1/p}\mrm{d}t\right]\\
				& \lesssim\int_0^\infty\left(\sum_{i=1}^n\mbb{P}(|\epsilon_i-\epsilon_i'|\geq t)\right)^{1/p}\mrm{d}t\\
				& \lesssim\int_0^\infty\left(\sum_{i=1}^n\mbb{P}(|\epsilon_i|\geq t)\right)^{1/p}\mrm{d}t.\\
				& =n^{1/p}\norm{\epsilon}_{p,1}.
			\end{align*}
			Therefore, if $\{k^{1-1/p}\rho_k^2\}$ is non-decreasing,
			$$
			\max_{1\leq k\leq n}\mbb{E}\left[\sup_{f\in\mcal{F}}\abs{k^{-1/p}\sum_{i=1}^k\xi_if(\mbf{X}_i)}\right]=n^{1-\frac{1}{p}}\rho_n^2,
			$$
			which implies that
			$$
			\mbb{E}\left[\sup_{f\in\mcal{F}}\abs{\sum_{i=1}^n\epsilon_if(\mbf{X}_i)}\right]	\lesssim n^{1/p}\norm{\epsilon}_{p,1}n^{1-\frac{1}{p}}\rho_n^2=n\rho_n^2.
			$$
		\end{enumerate}
	\end{proof}
	
	\begin{remark}
		If $\rho_k=k^{-\alpha}$, then $\{k\rho_k^2\}$ is non-decreasing when $\rho_k\geq k^{-1/2}$. $\{k^{1-\frac{1}{p}}\rho_k^2\}$ is non-decreasing when $\rho_k\geq k^{-\frac{1}{2}+\frac{1}{2p}}$. This shows that to obtain the desired result, the ``rate of convergence" $\rho_n$ should not converges to 0 too fast.
	\end{remark}
	
	Proposition \ref{Prop: Rate of Convergence Expectation of MP} can be used to obtain the rate of convergence of the multiplier process, which is just an straightforward application of Markov's inequality.
	
	\begin{prop}\label{Prop: Rate of Convergence of MP}
		Under the same assumption in Proposition \ref{Prop: 5HanWellner2019} and 
		$$
		\mbb{E}\left[\sup_{f\in\mcal{F}}\abs{\sum_{i=1}^k\xi_if(\mbf{X}_i)}\right]\lesssim k\rho_k^2\quad\mrm{ for all }k=1,\ldots,n.
		$$
		\begin{enumerate}[(i)]
			\item If $\norm{\epsilon}_p<\infty$ for some $p\geq1$ and the sequence $\{k\rho_k^2\}$ is non-decreasing, then
			$$
			\sup_{f\in\mcal{F}}\abs{\sum_{i=1}^n\epsilon_if(\mbf{X}_i)}=\mcal{O}_p\left(n^{1+\frac{1}{p}}\rho_n^2\right).
			$$
			
			\item If $\norm{\epsilon}_{p,1}<\infty$ for some $p\geq1$ and the sequence $\{k^{1-\frac{1}{p}}\rho_k^2\}$ is non-decreasing, then
			$$
			\sup_{f\in\mcal{F}}\abs{\sum_{i=1}^n\epsilon_if(\mbf{X}_i)}=\mcal{O}_p\left(n\rho_n^2\right).
			$$
		\end{enumerate}
	\end{prop}

	\newpage
	\subsection*{Auxiliary Results}
	\begin{prop}\label{Prop: sigmoid Coef}
		For all non-negative integer $m$,
		$$
		\sum_{a=1}^mC_a^{(m)}=m!.
		$$
	\end{prop}
	
	\begin{proof}
		We prove this result by induction. For $m=1$, the identity holds trivially according to the definition. Now, suppose that the result holds for $m$, then
		\begin{align*}
			\sum_{a=1}^{m+1}C_a^{(m+1)}	& =\sum_{a=1}^{m+1}aC_a^{(m)}+(m+2-a)C_a^{(m)}\\
			& \overset{(i)}{=}\sum_{a=1}^maC_a^{(m)}+\sum_{a'=1}^m(m+1-a')C_{a'}^{(m)}\\
			& =\sum_{a=1}^m(m+1)C_a^{(m)}\\
			& \overset{(ii)}{=}(m+1)m!\\
			& =(m+1)!,
		\end{align*}
		where in equation (i), we let $a'=a-1$, and equation (ii) follows from the induction hypothesis. Hence the desired result follows.
	\end{proof}
	
	\begin{lemma}\label{Lm: inprod}
		Suppose that $M:=\sup_{\mbf{x}\in\mcal{X}}|v^*(\mbf{x})|<\infty$, then
		$$
		\sup_{\substack{f\in\mcal{F}_n \\ \norm{f-f_0}\leq\rho_n}}\sqrt{n}\abs{\inprod{f-f_0}{v^*}_n-\inprod{f-f_0}{v^*}}=o_p(1).
		$$
		In particular, 
		$$
		\abs{\inprod{\hat{f}_n-f_0}{v^*}_n-\inprod{\hat{f}_n-f_0}{v^*}}=o_p(n^{-1/2}).
		$$
	\end{lemma}
	
	\begin{proof}
		Consider the function class
		$$
		\mcal{H}_n=\left\{(f-f_0)v^*:f\in\mcal{F}_n\right\}.
		$$
		Let $\{f_1,\ldots,f_N\}$ be a minimal $\epsilon$-cover of $\mcal{F}_n$ with respect to the $L_2(\mbb{P}_n)$-norm so that $N=N(\epsilon,\mcal{F}_n,L_2(\mbb{P}_n))$. Define
		$$
		h_j=(f_j-f_0)v^*,\quad j=1,\ldots,N.
		$$
		Note that for any $h\in\mcal{H}_n$, there exists $f\in\mcal{F}_n$ such that $h=(f-f_0)v^*$. For such a function $f$, we can find $j\in\{1,\ldots,N\}$ so that $\norm{f-f_j}_n<\epsilon$. Moreover, since
		\begin{align*}
			\norm{h-h_j}_n	& =\norm{(f-f_0)v^*-(f_j-f_0)v^*}_n\\
			& =\norm{(f-f_j)v^*}_n\\
			& \leq M \norm{f-f_j}_n\\
			& <M\epsilon.
		\end{align*}
		We know that $\{h_1,\ldots,h_N\}$ forms an $M\epsilon$-cover for $\mcal{H}_n$, and hence
		$$
		N(M\epsilon,\mcal{H}_n,L_2(\mbb{P}_n))\leq N(\epsilon,\mcal{F}_n,L_2(\mbb{P}_n)),
		$$
		which implies that $\log N(M\epsilon,\mcal{H}_n,L_2(\mbb{P}_n))\leq H(\epsilon)$ under (C1). On the other hand, since $\mcal{F}_n$ is uniformly bounded, we know that $B:=\sup_{x\in\mcal{X}}|f(x)|<\infty$ for any $f\in\mcal{F}_n$. Hence, for any $h\in\mcal{H}_n$,
		\begin{align*}
			\sup_{\mbf{x}\in\mcal{X}}|h(\mbf{x})|	& \leq\sup_{\mbf{x}\in\mcal{X}}|f(\mbf{x})-f_0(\mbf{x})|\sup_{\mbf{x}\in\mcal{X}}|v^*(\mbf{x})|\\
			& \leq(B+\sup_{\mbf{x}\in\mcal{X}}|f_0(\mbf{x})|)M<\infty,
		\end{align*}
		which implies that $\mcal{H}_n$ is uniformly bounded. It then follows from Theorem \citet{geer2000empirical} that $\mcal{H}_n$ is a Donsker class. Thus, from Lemma 2.3.11 in \citet{van1996weak}, for any sequence $\delta_n\to0$,
		$$
		\sup_{h\in\mcal{H}_{n,\delta_n}}\sqrt{n}\abs{(\mbb{P}_n-P)h}\xrightarrow{p}0,
		$$
		where $\mcal{H}_{n,\delta_n}=\left\{h_1-h_2:h_1,h_2\in\mcal{H}_n,\left(P(h_1-h_2-P(h_1-h_2))^2\right)^{1/2}\leq\delta_n\right\}$. For $f\in\mcal{F}_n$, set $h_1=(f-f_0)v^*$ and $h_2=(\pi_nf_0-f_0)v^*$. Note that
		\begin{align*}
			P(h_1-h_2-P(h_1-h_2))^2	& =P(h_1-h_2)^2-(P(h_1-h_2))^2\\
			& \leq P(h_1-h_2)^2\\
			&	\leq P(f-\pi_nf_0)^2M^2,
		\end{align*}
		we obtain
		\begin{align*}
			&	\sup_{\substack{f\in\mcal{F}_n \\ \norm{f-f_0}\leq\rho_n}}\sqrt{n}\abs{\inprod{f-f_0}{v^*}_n-\inprod{f-f_0}{v^*}}\\
			\leq	& \sup_{\substack{f\in\mcal{F}_n \\ \norm{f-\pi_nf_0}\leq\rho_n}}\sqrt{n}\abs{\inprod{f-f_0}{v^*}_n-\inprod{f-f_0}{v^*}}\\
			\leq	& \sup_{\substack{h_1\in\mcal{H}_n\\ \norm{h_1-h_2}\leq M\rho_n}}\sqrt{n}\abs{(\mbb{P}_n-P)(h_1-h_2)}+\sqrt{n}|(\mbb{P}_n-P)h_2|\\
			=	& o_p(1)+\sqrt{n}|(\mbb{P}_n-P)h_2|.
		\end{align*}
		Moreover, let $H(\mbf{x}_1,\ldots,\mbf{x}_n)=n^{-1}\sum_{i=1}^n(\pi_nf_0(\mbf{x}_i)-f_0(\mbf{x}_i))v^*(\mbf{x}_i)$ and note that
		\begin{align*}
			& \sup_{\mbf{x}_i,\mbf{x}_i'\in\mcal{X}}|H(\mbf{x}_1,\ldots,\mbf{x}_i,\ldots,\mbf{x}_n)-H(\mbf{x}_1,\ldots,\mbf{x}_i',\ldots,\mbf{x}_n)|\\ 
			=	& \sup_{\mbf{x}_i,\mbf{x}_i'\in\mcal{X}}\frac{1}{n}\abs{(\pi_nf_0-f_0)(\mbf{x}_i)v^*(\mbf{x}_i)-(\pi_nf_0-f_0)(\mbf{x}_i')v^*(\mbf{x}_i')}\\
			\leq & \frac{2M}{n}\sup_{\mbf{x}\in\mcal{X}}|\pi_nf_0(\mbf{x})-f_0(\mbf{x})|
		\end{align*}
		by McDiarmid inequality \citep{mcdiarmid1989method}, for all $t>0$,
		\begin{align*}
			\mbb{P}\left(\sqrt{n}|(\mbb{P}_n-P)h_2|>t\right)	& \lesssim 2\exp\left\{-\frac{t^2}{2M^2\left(\sup_{\mbf{x}\in\mcal{X}}|\pi_nf_0(\mbf{x})-f_0(\mbf{x})|\right)^2}\right\}\to0,
		\end{align*}
		which implies that $\sqrt{n}|(\mbb{P}_n-P)h_2|=o_p(1)$ and hence the desired claim follows.
	\end{proof}

\end{document}